\documentclass[letterpaper]{article} 
\usepackage[draft]{aaai25}  
\usepackage{times}  
\usepackage{helvet}  
\usepackage{courier}  
\usepackage[hyphens]{url}  
\usepackage{graphicx} 
\usepackage{natbib}  
\usepackage{caption} 
\usepackage{algorithm}
\usepackage{algorithmic}

\usepackage{newfloat}
\usepackage{listings}
\DeclareCaptionStyle{ruled}{labelfont=normalfont,labelsep=colon,strut=off} 
\floatstyle{ruled}
\newfloat{listing}{tb}{lst}{}
\floatname{listing}{Listing}
\usepackage{amsmath,amssymb,amsfonts}
\usepackage{algorithmic}
\usepackage{graphicx}
\usepackage{textcomp}
\usepackage{xcolor}
\def\BibTeX{{\rm B\kern-.05em{\sc i\kern-.025em b}\kern-.08em
    T\kern-.1667em\lower.7ex\hbox{E}\kern-.125emX}}

\usepackage[utf8]{inputenc} 
\usepackage[T1]{fontenc}    
\usepackage{url}            
\usepackage{booktabs}       
\usepackage{makecell}       
\usepackage{amsfonts}       
\usepackage{nicefrac}       
\usepackage{microtype}      
\usepackage{xcolor}         

\usepackage{mathtools}
\usepackage{amsthm}
\usepackage{thmtools}
\usepackage{thm-restate}

\usepackage[capitalize,noabbrev,nameinlink]{cleveref}

\usepackage[]{caption}
\usepackage{subcaption}
\usepackage{float}
\usepackage{graphicx,xcolor,graphics,url}
\usepackage{verbatim,ifthen,alltt,array}
\usepackage{latexsym,dsfont}
\usepackage{booktabs,multirow}
\usepackage{enumitem}
\usepackage{xparse}

\usepackage{xspace}

\usepackage{xfrac}
\usepackage[framemethod=TikZ]{mdframed}
\usepackage{nicefrac}       

\usepackage{stmaryrd}
\usepackage{textcomp}
\usepackage{siunitx}
\usepackage{xpatch}

\usepackage{pgf}
\usepackage{tikz}
\usepackage{forest}

\usetikzlibrary{arrows,decorations.pathmorphing,decorations.pathreplacing,decorations.footprints,fadings,calc,trees,mindmap,shadows,decorations.text,patterns,positioning,shapes,matrix,fit}
\usetikzlibrary{arrows,shadows,backgrounds}
\usetikzlibrary{arrows.meta}
\usetikzlibrary{positioning,fit}
\usetikzlibrary{automata}
\usetikzlibrary{matrix}
\usetikzlibrary{shapes.symbols,shapes.misc,shapes.arrows}
\usetikzlibrary{matrix,chains,scopes,decorations.pathmorphing}
\usetikzlibrary{bayesnet}
\usetikzlibrary{tikzmark}

\definecolor{darkred}{rgb}{0.7,0.1,0.1}
\definecolor{medred}{rgb}{0.5,0.1,0.1}
\definecolor{midred}{rgb}{0.7,0.2,0.2}
\definecolor{vdarkred}{rgb}{0.4,0.1,0.1}
\definecolor{darkslategray}{rgb}{0.18, 0.31, 0.31} 
\definecolor{platinum}{rgb}{0.9, 0.89, 0.89} 
\definecolor{gray}{rgb}{.4,.4,.4}
\definecolor{midgrey}{rgb}{0.5,0.5,0.5}
\definecolor{middarkgrey}{rgb}{0.35,0.35,0.35}
\definecolor{darkgrey}{rgb}{0.3,0.3,0.3}
\definecolor{darkred}{rgb}{0.7,0.1,0.1}
\definecolor{midblue}{rgb}{0.2,0.2,0.7}
\definecolor{darkblue}{rgb}{0.1,0.1,0.5}
\definecolor{darkgreen}{rgb}{0.1,0.5,0.1}
\definecolor{defseagreen}{cmyk}{0.69,0,0.50,0}
\definecolor{purple3}{RGB}{125,38,205}          

\definecolor{tyellow1}{HTML}{FCE94F}
\definecolor{tyellow2}{HTML}{EDD400}
\definecolor{tyellow3}{HTML}{C4A000}
\definecolor{torange1}{HTML}{FCAF3E}
\definecolor{torange2}{HTML}{F57900}
\definecolor{torange3}{HTML}{C35C00}
\definecolor{tbrown1}{HTML}{E9B96E}
\definecolor{tbrown2}{HTML}{C17D11}
\definecolor{tbrown3}{HTML}{8F5902}
\definecolor{tgreen1}{HTML}{8AE234}
\definecolor{tgreen2}{HTML}{73D216}
\definecolor{tgreen3}{HTML}{4E9A06}
\definecolor{tblue1}{HTML}{729FCF}
\definecolor{tblue2}{HTML}{3465A4}
\definecolor{tblue3}{HTML}{204A87}
\definecolor{tpurple1}{HTML}{AD7FA8}
\definecolor{tpurple2}{HTML}{75507B}
\definecolor{tpurple3}{HTML}{5C3566}
\definecolor{tred1}{HTML}{EF2929}
\definecolor{tred2}{HTML}{CC0000}
\definecolor{tred3}{HTML}{A40000}
\definecolor{tlgray1}{HTML}{EEEEEC}
\definecolor{tlgray2}{HTML}{D3D7CF}
\definecolor{tlgray3}{HTML}{BABDB6}
\definecolor{tdgray1}{HTML}{888A85}
\definecolor{tdgray2}{HTML}{555753}
\definecolor{tdgray3}{HTML}{2E3436}

\newcommand{\hlight}[1]{{\color{darkred}#1}}
\newcommand{\rhlight}[1]{\hlight{#1}}

\newcommand{\dghlight}[1]{{\color[RGB]{0,120,0}#1}}

\newcommand{\ncolor}[1]{{\color{tbrown3}#1}}

\newtheoremstyle{nthmstyle}
{3pt}
{3pt}
{}
{}
{\bfseries}
{.}
{.5em}
{}

\theoremstyle{nthmstyle}
\newtheorem{conjecture}{Conjecture}

\newtheorem{theorem}{Theorem}
\newtheorem{proposition}{Proposition}

\crefname{enumi}{}{}
\crefname{rstprop}{Proposition}{Propositions}

\newcommand{\fml}[1]{{\mathcal{#1}}}

\newcommand{\tn}[1]{\textnormal{#1}}

\newcommand{\msf}[1]{\ensuremath\mathsf{#1}}
\newcommand{\mbf}[1]{\ensuremath\mathbf{#1}}
\newcommand{\mbb}[1]{\ensuremath\mathbb{#1}}
\newcommand{\oper}[1]{\ensuremath\mathsf{#1}}

\newcommand{\waxp}{\ensuremath\mathsf{WAXp}}
\newcommand{\wcxp}{\ensuremath\mathsf{WCXp}}
\newcommand{\axp}{\ensuremath\mathsf{AXp}}
\newcommand{\cxp}{\ensuremath\mathsf{CXp}}
\newcommand{\aex}{\ensuremath\mathsf{AEx}}

\newcommand{\relevant}{\oper{Relevant}}
\newcommand{\irrelevant}{\oper{Irrelevant}}

\newcommand{\nfrac}{\nicefrac}

\newcommand{\exv}{\ensuremath\mathbf{E}}

\newcommand{\prob}{\ensuremath\mathbf{P}}

\newcommand{\cf}{\ensuremath\upsilon} 
\newcommand{\cfn}[1]{\ensuremath\upsilon_{#1}} 

\newcommand{\svn}[1]{\msf{Sc}_{#1}}

\newcommand{\sumd}{\Gamma}
\newcommand{\pnorm}[1]{\ensuremath{l}_{#1}}

\newcommand{\similar}{\ensuremath\sigma}
\newcommand{\tsimilar}{\ensuremath\mathsf{T}\sigma}

\newcommand{\pred}{\ensuremath\tn{P}}

\newcommand{\tm}{\scalebox{0.7125}[1.0]{\( - \)}}

\DeclareMathOperator*{\sv}{\msf{Sc}}
\DeclareMathOperator*{\argmax}{\tn{argmax}}

\definecolor{gray}{rgb}{.4,.4,.4}
\definecolor{midgrey}{rgb}{0.5,0.5,0.5}
\definecolor{middarkgrey}{rgb}{0.35,0.35,0.35}
\definecolor{darkgrey}{rgb}{0.3,0.3,0.3}
\definecolor{darkred}{rgb}{0.7,0.1,0.1}
\definecolor{midblue}{rgb}{0.2,0.2,0.7}
\definecolor{darkblue}{rgb}{0.1,0.1,0.5}
\definecolor{defseagreen}{cmyk}{0.69,0,0.50,0}
\newcommand{\jnote}[1]{\medskip\noindent$\llbracket$\textcolor{darkred}{joao}: \emph{\textcolor{middarkgrey}{#1}}$\rrbracket$\medskip}

\newcommand{\jnoteF}[1]{}

\newcounter{Comment}[Comment]
\DeclareMathOperator*{\limply}{\rightarrow}

\declaretheoremstyle[
  headfont=\bfseries,
  bodyfont=\itshape,
]{StdThmStyle}

\declaretheorem[name=Proposition,style=StdThmStyle]{rstprop}

\tikzset{
  0 my edge/.style={densely dashed, my edge},
  my edge/.style={-{Stealth[]}},
}

\setlist{nosep}

\begin{document}

\title{Towards trustable SHAP scores}

%

\author {
  Olivier L\'{e}toff\'{e}\textsuperscript{\rm 1},
  Xuanxiang Huang\textsuperscript{\rm 2},
  Joao Marques-Silva\textsuperscript{\rm 3}
}
\affiliations {
    \textsuperscript{\rm 1}Univ.~Toulouse, France\\
    \textsuperscript{\rm 2}CNRS@CREATE, Singapore\\
    \textsuperscript{\rm 3}ICREA, Univ.~Lleida, Spain\\
    olivier.letoffe@orange.fr,
    xuanxiang.huang.cs@gmail.com,
    jpms@icrea.cat
}

\maketitle

\begin{abstract}
  SHAP scores represent the proposed use of the well-known Shapley
  values in eXplainable Artificial Intelligence (XAI).
  Recent work has shown that the exact computation of SHAP scores can
  produce unsatisfactory results. Concretely, for some ML models, SHAP
  scores will mislead with respect to relative feature influence.
  To address these limitations, recently proposed alternatives exploit
  different axiomatic aggregations, all of which are defined in terms
  of abductive explanations. However, the proposed axiomatic
  aggregations are not Shapley values.
  %
  This paper investigates how SHAP scores can be modified so as to
  extend axiomatic aggregations to the case of Shapley values in XAI.
  More importantly, the proposed new definition of SHAP scores avoids
  all the known cases where unsatisfactory results have been
  identified.
  %
  The paper also characterizes the complexity of computing the
  novel definition of SHAP scores, highlighting families of
  classifiers for which computing these scores is tractable.
  Furthermore,
  the paper proposes modifications to the existing implementations of
  SHAP scores.
  These modifications eliminate some of the known limitations of SHAP
  scores, and have negligible impact in terms of performance.
\end{abstract}
%

%

\section{Introduction} \label{sec:intro}

Shapley values for eXplainable AI (XAI), i.e.~SHAP
scores~\cite{lundberg-nips17}, are arguably among the most widely used  
explainability methods that target the attribution of (relative)
feature importance, as exemplified by the success of the tool
SHAP~\cite{molnar-bk23,mishra-bk23}.\footnote{See~\url{https://github.com/shap/shap}.}
Despite the massive popularity of SHAP scores, some works have
identified limitations with their
use~\cite{shrapnel-corr19,friedler-icml20,najmi-icml20,taly-cdmake20,nguyen-ieee-access21,procaccia-aaai21,sharma-aies21,guigue-icml21,taly-uai21,friedler-nips21,roder-mlwa22}.
However, most of these limitations can be attributed to the results 
obtained with existing tools, and not necessarily with the theoretical
foundations of SHAP scores.
More recent work~\cite{hms-corr23a,hms-ijar24} uncovered examples of
classifiers where \emph{exact} SHAP scores assign manifestly
unsatisfactory importance to features. Namely, features having no
influence in a prediction can be assigned more importance than
features having the most influence in the prediction.
This recent evidence should be perceived as far more problematic,
because it reveals apparent limitations with the theoretical
foundations of SHAP scores, and not with concrete implementations.
%
%
Nevertheless, Shapley values are of fundamental importance, not only
in game theory~\cite{elkind-bk12}, but also in many other domains,
namely because of their intrinsic properties~\cite{shapley-ctg53}.
%
As a result, a natural question is whether the definitions of Shapley
values in XAI can be changed, so as to avoid situations where the
computed feature importance is problematic.

\paragraph{Contributions.}
%
This paper argues that the key issue with SHAP scores is not the 
use of Shapley values in explainability per se, and shows that the
identified shortcomings of SHAP scores can be solely attributed to the 
characteristic functions used in earlier
works~\cite{kononenko-jmlr10,kononenko-kis14,lundberg-nips17,blockbaum-aistats20,najmi-icml20,barcelo-aaai21,vandenbroeck-aaai21,vandenbroeck-jair22,barcelo-jmlr23}. 
As noted in the recent past~\cite{blockbaum-aistats20,najmi-icml20}, by
changing the characteristic function, one is able to produce different
sets of SHAP scores\footnote{%
Unfortunately, this paper argues that past alternative proposals of
characteristic functions~\cite{blockbaum-aistats20,najmi-icml20} also
exhibit critical limitations.}.
Motivated by these observations, the paper outlines fundamental
properties that characteristic functions ought to exhibit in the
context of XAI.
Furthermore, the paper proposes several novel characteristic
functions, which either respect some or all of the identified
properties.
In addition, the paper analyzes the impact of the novel characteristic
functions on the computational complexity of computing SHAP scores, by
building on recent work on the same
topic~\cite{vandenbroeck-aaai21,barcelo-aaai21,vandenbroeck-jair22,barcelo-jmlr23}.
An indirect consequence of our work is that \emph{corrected} SHAP
scores can be safely used for feature attribution in XAI, while
offering strong guarantees regarding known shortcomings.


\paragraph{Related work.}
%
SHAP scores are ubiquitously used in
XAI~\cite{lundberg-nips17,molnar-bk23,mishra-bk23}.
Recent work argues that the existing definitions of (exact) SHAP
scores can yield unsatisfactory
results~\cite{hms-corr23a,hms-ijar24}. Motivated by these results,
different works proposed alternative solutions to the use of SHAP
scores~\cite{ignatiev-corr23a,izza-corr23,ignatiev-corr23b,izza-aaai24,ignatiev-sat24}.
Furthermore, one of these solutions~\cite{izza-aaai24} investigates
the use of power indices from the field of a priori voting
power~\cite{machover-hscv15} as a solution for feature importance in
XAI, covering several well-known power-indices.
However, Shapley values are also at the core of the well-known
Shapley-Shubik power index~\cite{shapley-apsr54}, one of the
best-known power indices, and which is not studied
in~\cite{izza-aaai24}. Thus, an open question is how to extend the
recent work on power indices for XAI~\cite{izza-aaai24} to the case of
the Shapley-Shubik index.

\jnoteF{Comments on~\cite{izza-aaai24}}

\paragraph{Organization.}
%
%
The paper starts by introducing the notation and definitions used
throughout the paper.
Afterwards, the paper briefly dissects some of the recently reported
shortcomings with SHAP scores~\cite{msh-cacm24,hms-ijar24,hms-corr23a}.
Motivated by those shortcomings, the paper then proposes properties
that characteristic functions should exhibit.
The paper then proposes several novel characteristic functions, which
are shown to correct some or all of the shortcomings of the
characteristic functions used in earlier work.
Next, the paper studies the complexity of computing SHAP scores given
the novel characteristic functions proposed in this paper.
The paper also outlines a simple modification to the SHAP
tool~\cite{lundberg-nips17}, which corrects some of the shortcomings
of SHAP scores.


\section{Preliminaries} \label{sec:prelim}

\paragraph{Classification \& regression problems.}
%
Let $\fml{F}=\{1,\ldots,m\}$ denote a set of features.
Each feature $i\in\fml{F}$ takes values from a domain $\mbb{D}_i$.
Domains can be categorical or ordinal. If ordinal, domains can be
discrete or real-valued.
Feature space is defined by
$\mbb{F}=\mbb{D}_1\times\mbb{D}_2\times\ldots\times\mbb{D}_m$. 
Throughout the paper domains are assumed to be discrete-valued.%
\footnote{
The results in the paper can be generalized to continuous-valued
features. For real-valued features, the only changes involve the
definition of expected value and probability, respectively
in~\eqref{eq:evdef} and~\eqref{eq:probdef}.}
The notation $\mbf{x}=(x_1,\ldots,x_m)$ denotes an arbitrary point in 
feature space, where each $x_i$ is a variable taking values from
$\mbb{D}_i$. Moreover, the notation $\mbf{v}=(v_1,\ldots,v_m)$
represents a specific point in feature space, where each $v_i$ is a
constant representing one concrete value from $\mbb{D}_i$.
A classifier maps each point in feature space to a class taken from
$\fml{K}=\{c_1,c_2,\ldots,c_K\}$. Classes can also be categorical or
ordinal. However, and unless otherwise stated, classes are assumed to
be ordinal.
In the case of regression, each point in feature space is mapped to an
ordinal value taken from a set $\mbb{C}$, e.g.\ $\mbb{C}$ could denote
$\mbb{Z}$ or $\mbb{R}$.
Therefore, a classifier $\fml{M}_{C}$ is characterized by a
non-constant \emph{classification function} $\kappa$ that maps feature
space $\mbb{F}$ into the set of classes $\fml{K}$,
i.e.\ $\kappa:\mbb{F}\to\fml{K}$.
A regression model $\fml{M}_R$ is characterized by a non-constant
\emph{regression function} $\rho$ that maps feature space $\mbb{F}$
into the set elements from $\mbb{C}$, i.e.\ $\rho:\mbb{F}\to\mbb{C}$. 
A classifier model $\fml{M}_{C}$ is represented by a tuple
$(\fml{F},\mbb{F},\fml{K},\kappa)$, whereas a regression model
$\fml{M}_{R}$ is represented by a tuple
$(\fml{F},\mbb{F},\mbb{C},\rho)$.%
\footnote{As shown in the extended version of this
paper~\cite{lhms-corr24}, other ML models can also be represented with
similar ideas. This is the case with the use of a softmax layer in
neural networks (NNs).}
When viable, we will represent an ML model $\fml{M}$ by a tuple
$(\fml{F},\mbb{F},\mbb{T},\tau)$, with $\tau:\mbb{F}\to\mbb{T}$,
without specifying whether whether $\fml{M}$ denotes a classification
or a regression model.
A \emph{sample} (or instance) denotes a pair $(\mbf{v},q)$, where
$\mbf{v}\in\mbb{F}$ and either $q\in\fml{K}$, with
$q=\kappa(\mbf{v})$, or $q\in\mbb{C}$, with $q=\rho(\mbf{v})$.

\paragraph{Additional notation.}
%
%
An explanation problem is a tuple $\fml{E}=(\fml{M},(\mbf{v},q))$,
where $\fml{M}$ can either be a classification or a regression model,
and $(\mbf{v},q)$ is a target sample, with
$\mbf{v}\in\mbb{F}$.
(Observe that $q=\kappa(\mbf{v})$, with
$q\in\fml{K}$, in the case of a classification model, and  
$q=\rho(\mbf{v})$, with $q\in\mbb{C}$, in the case of a regression
model.)

%
Given $\mbf{x},\mbf{v}\in\mbb{F}$, and $\fml{S}\subseteq\fml{F}$, the
predicate $\mbf{x}_{\fml{S}}=\mbf{v}_{\fml{S}}$ is defined as follows:
\[
\mbf{x}_{\fml{S}}=\mbf{v}_{\fml{S}} ~~ := ~~ \left(\bigwedge\nolimits_{i\in\fml{S}}x_i=v_i\right)
\]
The set of points for which $\mbf{x}_{\fml{S}}=\mbf{v}_{\fml{S}}$ is
defined by
$\Upsilon(\fml{S};\mbf{v})=\{\mbf{x}\in\mbb{F}\,|\,\mbf{x}_{\fml{S}}=\mbf{v}_{\fml{S}}\}$.

\paragraph{Distributions, expected value.}
Throughout the paper, it is assumed a \emph{uniform probability
distribution} on each feature, and such that all features are
independent.
%
Thus, the probability of an arbitrary point in feature space
becomes:
%
\begin{equation}
  \prob(\mbf{x}) := \sfrac{1}{\Pi_{i\in\fml{F}}|\mbb{D}_i|}
\end{equation}
That is, every point in the feature space has the same probability.
The \emph{expected value} of an ML model $\tau:\mbb{F}\to\mbb{T}$
is denoted by $\mbf{E}[\tau]$. 
%
Furthermore, let
$\exv[\tau(\mbf{x})\,|\,\mbf{x}_{\fml{S}}=\mbf{v}_{\fml{S}}]$
represent the expected value of $\tau$ over points in feature space
consistent with the coordinates of $\mbf{v}$ dictated by $\fml{S}$,
which is defined as follows:
%
%
\begin{equation} \label{eq:evdef}
  \exv[\tau(\mbf{x})\,|\,\mbf{x}_{\fml{S}}=\mbf{v}_{\fml{S}}]
  :=\sfrac{1}{|\Upsilon(\fml{S};\mbf{v})|}
  \sum\nolimits_{\mbf{x}\in\Upsilon(\fml{S};\mbf{v})}\tau(\mbf{x})
\end{equation}
%
%
Similarly, we define,
\begin{equation} \label{eq:probdef}
  \prob(\pi(\mbf{x})\,|\,\mbf{x}_{\fml{S}}=\mbf{v}_{\fml{S}}):=
  \sfrac{1}{|\Upsilon(\fml{S};\mbf{v})|}
  \sum\nolimits_{\mbf{x}\in\Upsilon(\fml{S};\mbf{v})}\tn{ITE}(\pi(\mbf{x}),1,0)
\end{equation}
where $\pi:\mbb{F}\to\{0,1\}$ is some predicate. 

\jnoteF{
  Given $\mbf{z}\in\mbb{F}$ and$\fml{S}\subseteq\fml{F}$, let
  $\mbf{z}_{\fml{S}}$ represent the vector composed of the coordinates
  of $\mbf{z}$ dictated by $\fml{S}$.
  \[
  \exv{\kappa\,|\,\mbf{x}_{\fml{S}}=\mbf{v}_{\fml{S}}}
  :=\frac{1}{|\Upsilon(\fml{S};\mbf{v})|}
  \sum\nolimits_{\mbf{x}\in\Upsilon(\fml{S};\mbf{v})}\kappa(\mbf{x})
  \]
}

\paragraph{Shapley values.}
%
Shapley values were proposed in the context of game theory in the
early 1950s by L.\ S.\ Shapley~\cite{shapley-ctg53}. Shapley values
were defined given some set $\fml{N}$, and a \emph{characteristic
function}, i.e.\ a real-valued function defined on the subsets of
$\fml{N}$, $\cf:2^{\fml{N}}\to\mbb{R}$.%
\footnote{%
The original formulation also required super-additivity of the
characteristic function, but that condition has been relaxed in more
recent works~\cite{dubey-ijgt75,young-ijgt85}.}.
It is well-known that Shapley values represent the \emph{unique}
function that, given $\fml{N}$ and $\cf$, respects a number of
important axioms. More detail about Shapley values is available in
standard
references~\cite{shapley-ctg53,dubey-ijgt75,young-ijgt85,roth-bk88}.
Besides the recent uses in XAI, Shapley values have been used for
assigning measures of relative importance in 
computational social choice~\cite{elkind-bk12} (including a priori
voting power~\cite{shapley-apsr54,machover-bk98}), 
measurement of inconsistency in knowledge
bases~\cite{hunter-kr06,hunter-aij10},
and
intensity of attacks in argumentation
frameworks~\cite{amgoud-ijcai17}.

\paragraph{SHAP scores.}
%
In the context of explainability, Shapley values are most often
referred to as SHAP scores%
~\cite{kononenko-jmlr10,kononenko-kis14,lundberg-nips17,barcelo-aaai21,barcelo-jmlr23},
and consider a specific characteristic function
$\cf_e:2^{\fml{F}}\to\mbb{R}$,
which is defined by,
\begin{equation} \label{eq:cfs}
  \cf_e(\fml{S};\fml{E}) :=
  \exv[\tau(\mbf{x})\,|\,\mbf{x}_{\fml{S}}=\mbf{v}_{\fml{S}}]
\end{equation}
%
%
%
Thus, given a set $\fml{S}$ of features,
$\cf_e(\fml{S};\fml{E})$ represents the \emph{e}xpected value
of the classifier over the points of feature space represented by
$\Upsilon(\fml{S};\mbf{v})$.
%
The formulation presented in earlier
work~\citep{barcelo-aaai21,barcelo-jmlr23} allows for different input
distributions when computing the average values. For the purposes of
this paper, it suffices to consider solely a uniform input
distribution, and so the dependency on the input distribution is not
accounted for.
Independently of the distribution considered, it should be clear that
in most cases $\cfn{e}(\emptyset)\not=0$; this is the case for example
with boolean classifiers~\cite{barcelo-aaai21,barcelo-jmlr23}.

To simplify the notation, the following definitions are used,
\begin{align}
  \Delta_i(\fml{S};\fml{E},\cf) & :=
  \left(\cf(\fml{S}\cup\{i\})-\cf(\fml{S})\right)
  \label{eq:def:delta}
  \\[2pt] 
  \varsigma(|\fml{S}|) & :=
  \sfrac{|\fml{S}|!(|\fml{F}|-|\fml{S}|-1)!}{|\fml{F}|!} 
  \label{eq:def:vsigma}
\end{align}
(Observe that $\Delta_i$ is parameterized on $\fml{E}$ and $\cf$.)

Finally, let $\svn{E}:\fml{F}\to\mbb{R}$, i.e.\ the SHAP score for
feature $i$, be defined by,\footnote{%
Throughout the paper, the definitions of $\Delta_i$ and $\sv$ are
explicitly associated with the characteristic function used in their
definition.}.
\begin{equation} \label{eq:sv}
  \svn{E}(i;\fml{E},\cfn{e}):=\sum\nolimits_{\fml{S}\subseteq(\fml{F}\setminus\{i\})}\varsigma(|\fml{S}|)\times\Delta_i(\fml{S};\fml{E},\cfn{e}) 
\end{equation}
Given a sample $(\mbf{v},q)$, the SHAP score assigned to each
feature measures the \emph{contribution} of that feature with respect
to the prediction. 
From earlier work, it is understood that a positive/negative value
indicates that the feature can contribute to changing the prediction,
whereas a value of 0 indicates no
contribution~\citep{kononenko-jmlr10}.
%


\paragraph{Similarity predicate.}
%
Given an ML model and some input $\mbf{x}$, the output of the ML model
is \emph{distinguishable} with respect to the sample $(\mbf{v},q)$ if the
observed change in the model's output is deemed sufficient; 
otherwise it is \emph{similar} (or indistinguishable).
This is represented by a \emph{similarity} predicate (which will also
be viewed as a boolean function) 
$\similar:\mbb{F}\to\{\bot,\top\}$ (where $\bot$ signifies
\emph{false}, and $\top$ signifies \emph{true}).%
\footnote{%
For simplicity, and with a minor abuse of notation, when $\similar$
is used in a scalar context, it is interpreted as a boolean function,
i.e.\ $\similar:\mbb{F}\to\{0,1\}$, with 0 replacing $\bot$ and 1 
replacing $\top$.}
Concretely,
$\similar(\mbf{x};\fml{E})$ holds true iff the change in the ML model
output is deemed \emph{insufficient} and so no observable difference
exists between the ML model's output for $\mbf{x}$ and $\mbf{v}$.%
\footnote{
Throughout the paper, parameterization are shown after the separator
';', and will be elided when clear from the context.}
For regression problems, we write instead $\similar$ as the
instantiation of a template predicate,
i.e.\ $\similar(\mbf{x};\fml{E})=\tsimilar(\mbf{x};\fml{E},\delta)$,
where $\delta$ is an optional measure of output change, which can be
set to 0.%
\footnote{%
Exploiting a threshold to decide whether there exists an observable
change has been used in the context of adversarial
robustness~\cite{barrett-nips23}. Furthermore, the relationship
between adversarial examples and explanations is
well-known~\cite{inms-nips19,barrett-nips23}.}
%
Given a change in the input from $\mbf{v}$ to $\mbf{x}$, a change in
the output is indistinguishable (i.e.\ the outputs are similar) if,
\[
\similar(\mbf{x};\fml{E}) 
:= \tsimilar(\mbf{x};\fml{E},\delta)
:= [|\rho(\mbf{x})-\rho(\mbf{v})|\le\delta]
\]
otherwise, it is distinguishable.

For classification problems, similarity is defined to equate with not
changing the predicted class. Given a change in the input from 
$\mbf{v}$ to $\mbf{x}$, a change in the output is indistinguishable
(i.e.\ the outputs are similar) if,
\[
\similar(\mbf{x};\fml{E}):=[\kappa(\mbf{x})=\kappa(\mbf{v})]
\]
otherwise, it is distinguishable. (As shown in the remainder of this
paper, $\similar$ allows abstracting away whether the underlying
model implements classification or regression.)

%
It will be helpful to list a few properties of $\similar$.
Observe that
$\forall(\fml{A}\subseteq\mbb{F}).[\exv[\similar(\mbf{x};\fml{E})\,|\,\mbf{x}\in\fml{A}]\in[0,1]]$.
It is also plain to conclude that for
$\fml{A},\fml{B}\subseteq\mbb{F}$, with $\fml{A}\subseteq\fml{B}$, and
given $u\in\{0,1\}$, 
if $\exv[\similar(\mbf{x};\fml{E})\,|\,\mbf{x}\in\fml{B}]=u$,
then
$\exv[\similar(\mbf{x};\fml{E})\,|\,\mbf{x}\in\fml{A}]=u$,
A few more properties of $\similar$ are apparent. For
$\fml{A}\subseteq\mbb{F}$, $u\in\{0,1\}$,
$\left(\exv[\similar(\mbf{x};\fml{E})\,|\,\mbf{x}\in\fml{A}]=u\right)\leftrightarrow\forall(\mbf{x}\in\fml{A}).[\similar(\mbf{x};\fml{E})=u]$.
As a result, it is also the case that
$\left(\exv[\similar(\mbf{x};\fml{E})\,|\,\mbf{x}\in\fml{A}]<1\right)\leftrightarrow\exists(\mbf{x}\in\fml{A}).[\similar(\mbf{x};\fml{E})=0]$.

\paragraph{Adversarial examples.}
%
Adversarial examples serve to reveal the brittleness of ML
models~\cite{szegedy-iclr14,szegedy-iclr15}. Adversarial robustness
indicates the absence of adversarial examples. The importance of
deciding adversarial robustness is illustrated by a wealth of
competiting alternatives~\cite{johnson-sttt23}.

Given a sample $(\mbf{v},q)$, and a norm $l_p$, a point
$\mbf{x}\in\mbb{F}$ is an \emph{adversarial example} if the prediction
for $\mbf{x}$ is distinguishable from that for $\mbf{v}$. Formally, we
write,
\[
\aex(\mbf{x};\fml{E}) ~~ := ~~
\left(||\mbf{x}-\mbf{v}||_{p}\le\epsilon\right)\land
\neg\similar(\mbf{x};\fml{E})
\]
where the $l_p$ distance between the given point $\mbf{v}$ and other
points of interest is restricted to $\epsilon>0$.
%
%
Moreover, we define a \emph{constrained} adversarial example, such
that the allowed set of points is given by the predicate
$\mbf{x}_{\fml{S}}=\mbf{v}_{\fml{S}}$. Thus,
\[
\aex(\mbf{x},\fml{S};\fml{E}) ~~ := ~~
\left(||\mbf{x}-\mbf{v}||_{p}\le\epsilon\right)\land
\left(\mbf{x}_{\fml{S}}=\mbf{v}_{\fml{S}}\right)\land
\neg\similar(\mbf{x};\fml{E})
\]
Adversarial robustness is concerned with determining whether complex
ML models do not exhibit adversarial examples for chosen samples.

\paragraph{Abductive and contrastive explanations (AXps/CXps).}
AXps and CXps are examples of formal
explanations for classification
problems~\cite{inms-aaai19,msi-aaai22,ms-rw22,darwiche-lics23}.
We adopt a generalization that encompasses regression
problems~\cite{ms-isola24}. (Although we define abductive/contrastive
explanations in terms of probabilities, these can be rewritten using
expected values. In addition, the proposed definitions are equivalent
to the logic-based formulations proposed in other
works~\cite{ms-rw22}.) 

A weak abductive explanation (WAXp) denotes a set of features
$\fml{S}\subseteq\fml{F}$, such that for every point in feature space
the ML model output is similar to the given sample: $(\mbf{v},q)$.
The condition for a set of features to represent a WAXp (which also
defines a corresponding predicate $\waxp$) is as follows:
\[
\waxp(\fml{S};\fml{E}) ~~ := ~~
\prob(\similar(\mbf{x};\fml{E})\,|\,\mbf{x}_{\fml{S}}=\mbf{v}_{\fml{S}})
= 1
\]
%
%
Moreover, an AXp is a subset-minimal WAXp.

A weak contrastive explanation (WCXp) denotes a set of features
$\fml{S}\subseteq\fml{F}$, such that there exists some point in
feature space, where only the features in $\fml{S}$ are allowed to
change, that makes the ML model output distinguishable from the given
sample $(\mbf{v},q)$.
The condition for a set of features to represent a WCXp (which also
defines a corresponding predicate $\wcxp$) is as follows:
%
\[
\wcxp(\fml{S};\fml{E}) ~~ := ~~
\prob(\similar(\mbf{x};\fml{E})\,|\,\mbf{x}_{\fml{F}\setminus\fml{S}}=\mbf{v}_{\fml{F}\setminus\fml{S}})
< 1
\]
%
%
Moreover, a CXp is a subset-minimal WCXp.

One immediate observation is that each WAXp is a hitting set (HS) of
the set of WCXps, and each WCXp is a HS of the set of WAXps.
Furthermore, one can prove that a set of features is an AXp iff it is
a minimal hitting set (MHS) of the set of CXps, and
vice-versa~\cite{msi-aaai22}.
(Although this result has been proved for classification problems, our
proposed framework generalizes the result also to regression problems.)


Given the previous definitions, it is plain the following result.
\begin{rstprop}
  $\exists(\mbf{x}\in\mbb{F}).\left(\aex(\mbf{x},\fml{F}\setminus\fml{S};\fml{E})\right)$
  iff
  $\wcxp(\fml{S};\fml{E})$, i.e.\ there exists a constrained
  adversarial example with the features $\fml{F}\setminus\fml{S}$ iff
  the set $\fml{S}$ is a weak CXp.
\end{rstprop}

\paragraph{Feature (ir)relevancy.}
%
The set of features that are included in at least one (abductive) 
explanation is defined as follows:
\begin{equation}
  \mathfrak{F}(\fml{E}):=\{i\in\fml{X}\,|\,\fml{X}\in2^{\fml{F}}\land\axp(\fml{X})\}
\end{equation}
where predicate $\axp(\fml{X})$ holds true iff $\fml{X}$ is an AXp.
(A well-known result is that $\mathfrak{F}(\fml{E})$ remains unchanged
if CXps are used instead of AXps~\cite{msi-aaai22}, in which case
predicate $\cxp(\fml{X})$ holds true iff $\fml{X}$ is a CXp.) 
%
%
%
%
%
Finally, a feature $i\in\fml{F}$ is \emph{irrelevant}, i.e.\ predicate
$\irrelevant(i)$ holds true, if $i\not\in\mathfrak{F}(\fml{E})$;
otherwise feature $i$ is \emph{relevant}, and predicate $\relevant(i)$
holds true. 
Clearly, given some explanation problem $\fml{E}$,
$\forall(i\in\fml{F}).\irrelevant(i)\leftrightarrow\neg\relevant(i)$.

\section{Unsatisfactory SHAP Scores} \label{sec:issues}

Recent work~\cite{hms-corr23a,hms-ijar24} reports examples of
classifiers for which the SHAP scores are patently unsatisfactory.
This section shows that similar unsatisfactory scores can be obtained
with regression models. 

\begin{figure*}[t]
  \begin{subfigure}[b]{0.35\linewidth} 

    \begin{tabular}{c} ~~\\[5pt] \end{tabular}

    \centering
    \scalebox{0.875}{
      \renewcommand{\arraystretch}{1.15}
      \renewcommand{\tabcolsep}{0.575em}
      \renewcommand{\tabcolsep}{0.45em}
\begin{tabular}{cccc} \toprule
  row & $x_1$ & $x_2$ & $\rho(\mbf{x})$ 
  \\ \toprule
  1 & 0 & 0 & $\nfrac{-1}{2}$ \\
  2 & 0 & 1 & $\nfrac{3}{2}$ \\
  3 & 1 & 0 & 1 \\
  \tikzmarknode{a}{4} & 1 & 1               & \tikzmarknode{b}{1}
  \\ \bottomrule
  \begin{tikzpicture}[overlay,remember picture]
    \node[draw=midblue, thin, xshift=-0.35pt, yshift=-0.35pt, inner
      sep=2.0pt, fit=(a) (b)] {};
  \end{tikzpicture}
\end{tabular}

    }
    %
    %
    \caption{Tabular representation (TR)}
  \end{subfigure}
  \begin{subfigure}[b]{0.3\linewidth} 
    \centering
    \scalebox{0.85}{
%
\forestset{
  BDT/.style={
    for tree={
      l=1.5cm,s sep=1.15cm,
      if n children=0{}{circle}, 
      draw=midblue,
      text=midblue,
      edge={
        my edge
      },
      edge=thick,
    }
  },
}
\begin{forest}
  BDT
  [{$x_1$}, label={[yshift=-6.875ex]{{\tiny1}}} 
    [{$x_2$}, label={[yshift=-6.875ex]{{\tiny2}}}, 
      edge label={node[midway,left,xshift=-1.5pt] {{\scriptsize$\in\{0\}$}}}
      [\ncolor{\boldmath{$\nfrac{-1}{2}$}}, label={[yshift=-5.75ex]{{\tiny4}}},
        edge label={node[midway,left,xshift=-0.5pt]
          {{\scriptsize$\in\{0\}$}}}, rectangle, fill={torange1!10} ]
      [\ncolor{\boldmath{$\nfrac{3}{2}$}}, label={[yshift=-5.75ex]{{\tiny5}}},
        edge label={node[midway,right,xshift=-0.575pt] {{\scriptsize$\in\{1\}$}}}, rectangle, fill={torange1!10} ]
    ]
    [\ncolor{\textbf{1}}, label={[yshift=-5.25ex]{{\tiny3}}},
      edge={very thick,draw=tblue2}, edge label={node[midway,right,xshift=0.5pt] {{\scriptsize$\in\{1\}$}}},
      rectangle, fill={torange1!10} ]
  ]
\end{forest}
    }
    \caption{Regression tree (RT)}
  \end{subfigure}
  \begin{subfigure}[b]{0.35\linewidth} 

    \begin{tabular}{c} ~~\\[5pt] \end{tabular}

    \centering
    \scalebox{0.875}{
      \renewcommand{\arraystretch}{1.15}
      \renewcommand{\tabcolsep}{0.575em}
      \renewcommand{\tabcolsep}{0.45em}
\begin{tabular}{ccc} \toprule
  $\fml{S}$ & $\msf{rows}(\fml{S})$ & $\cfn{e}(\fml{S})$
  \\ \toprule
  $\emptyset$ & $1,2,3,4$ & $\nfrac{3}{4}$ \\
  $\{1\}$ & $3,4$ & $1$ \\
  $\{2\}$ & $2,4$ & $\nfrac{5}{4}$ \\
  $\{1,2\}$ & $4$ & $1$
  \\ \bottomrule
\end{tabular}

    }

    \begin{tabular}{c} ~~\\[-1.0pt] \end{tabular}

    \caption{Expected values} \label{ex:tr:rt:avg}
  \end{subfigure}
  \caption{Simple regression tree model, 
    adapted from~\cite[Fig.~8.1]{james-bk17}. 
    The target sample is $((1,1),1)$.}
  \label{ex:tr:rt}
\end{figure*}

\paragraph{Case study -- an example regression model.}
%
\cref{ex:tr:rt} shows a regression tree (RT)~\cite{breiman-bk84} used
as the running example for the remainder of the paper. (The RT is
adapted from~\cite[Fig~8.1]{james-bk17}, with the attribute names
simplified, and with the predicted values changed.)
It is plain to conclude that $\fml{F}=\{1,2\}$,
$\mbb{D}_1=\mbb{D}_2=\{0,1\}$, $\mbb{F}=\{0,1\}^2$,
$\mbb{C}=\mbb{R}$,
and that the regression function is given either by the tabular
representation (TR) or by the regression table (RT).
%
In addition, the target instance is $((1,1),1)$.
\cref{ex:tr:rt:avg} shows the values of $\cfn{e}$ for the possible
values of set $\fml{S}$.
Furthermore, in the remainder of the paper, it is assumed that the
value used to instantiate $\similar$ for this regression problem is
$\delta=0.5$.
%

\paragraph{Adversarial examples for the case study.}
%
As can be observed in~\cref{ex:calcs:aex}, for an input to result in a
distinguishable output, feature 1 must be changed. In contrast,
feature 2 need not be changed. Any subset-minimal set of features that
must be changed to make the output distinguishable only includes
feature 1.

\paragraph{Explanations for the case study.}
%
The computation of abductive explanations is summarized
in~\cref{ex:calcs:axp}.
As can be observed, the set of AXps contains    
only $\{1\}$. Hence, by minimal-hitting set duality, the set of CXps
is also $\{\{1\}\}$.
The conclusions are that: i) if feature 1 is fixed, then the
prediction cannot be changed; and ii) if the prediction is to be
changed, then feature 1 must be changed.
These conclusions are aligned with the conclusions obtained from
analyzing the adversarial examples.

\begin{figure*}[t]
    \begin{subfigure}[c]{0.25\linewidth} 
      \centering
      \scalebox{0.875}{
        \renewcommand{\arraystretch}{1.15}
        \renewcommand{\tabcolsep}{0.575em}
        \begin{tabular}{cccc} \toprule
  $x_1$ & $x_2$ & $\msf{AEx}(\mbf{x})$ & $||\cdot||_{0}$ \\
  \toprule
  0 & 0 & Y & 2 \\
  0 & 1 & Y & 1 \\
  1 & 0 & N & 1 \\
  1 & 1 & N & 0 \\
  \bottomrule
\end{tabular}

      }
      \caption{Adversarial examples}\label{ex:calcs:aex}
    \end{subfigure}
    %
    %
    \begin{subfigure}[c]{0.25\linewidth} 
      \centering
      \scalebox{0.875}{
        \renewcommand{\arraystretch}{1.15}
        \renewcommand{\tabcolsep}{0.575em}
        \begin{tabular}{ccc} \toprule
  $\fml{S}$ & $\waxp(\fml{S})$ & $\axp(\fml{S})$ \\
  \toprule
  $\emptyset$ & N & -- \\
  $\{1\}$ & Y & Y \\
  $\{2\}$ & N & -- \\
  $\{1,2\}$ & Y & N \\
  \bottomrule
\end{tabular}

      }
      \caption{Abductive explanations}\label{ex:calcs:axp}
    \end{subfigure}
  %
    \begin{subfigure}[c]{0.5\linewidth} 
      \centering
      \scalebox{0.875}{
        \renewcommand{\arraystretch}{1.15}
        \renewcommand{\tabcolsep}{0.5em}
        \begin{tabular}{cccccc} \toprule
  \multicolumn{6}{c}{$i=1$} \\
  \toprule
  $\fml{S}$ & $\cf_e(\fml{S})$ & $\cf_e(\fml{S}\cup\{1\})$ &
  $\Delta_1(\fml{S})$ & $\varsigma(\fml{S})$ &
  $\varsigma(\fml{S})\times\Delta_1(\fml{S})$ \\
  \toprule
  $\emptyset$ & $\nfrac{3}{4}$ & 1 & $\nfrac{1}{4}$ & $\nfrac{1}{2}$ &
  $\nfrac{1}{8}$ \\
  $\{2\}$ & $\nfrac{5}{4}$ & 1 & $\nfrac{-1}{4}$ & $\nfrac{1}{2}$ &
  $\nfrac{-1}{8}$ \\
  \midrule
  \multicolumn{5}{r}{$\svn{E}(1)~~=$} & \multicolumn{1}{c}{0} \\
  \toprule
  \multicolumn{6}{c}{$i=2$} \\
  \toprule
  $\fml{S}$ & $\cfn{e}(\fml{S})$ & $\cfn{e}(\fml{S}\cup\{2\})$ &
  $\Delta_2(\fml{S})$ & $\varsigma(\fml{S})$ &
  $\varsigma(\fml{S})\times\Delta_2(\fml{S})$ \\
  \toprule
  $\emptyset$ & $\nfrac{3}{4}$ & $\nfrac{5}{4}$ & $\nfrac{1}{2}$ & $\nfrac{1}{2}$
  & $\nfrac{1}{4}$ \\
  $\{1\}$ & 1 & 1 & 0 & $\nfrac{1}{2}$ & 0 \\
  \midrule
  \multicolumn{5}{r}{$\svn{E}(2)~~=$} & \multicolumn{1}{c}{$\nfrac{1}{4}$} \\
  \bottomrule
\end{tabular}

      }
      %
      %
      \caption{SHAP scores}\label{ex:calcs:svs}
    \end{subfigure}
  %
  \caption{AExs, AXPs \& SHAP scores for the regression tree
    from~\cref{ex:tr:rt} and target sample $((1,1),1)$.
    For simplicity, parameterizations are elided.}
  \label{ex:calcs}
\end{figure*}

\paragraph{SHAP scores for the case study.}
The computation of exact SHAP scores is
summarized~\cref{ex:calcs:svs}, using the definitions introduced
earlier in the paper.
As can be observed, the SHAP score for feature 1 is 0, and the SHAP
score for feature 2 is $\nfrac{1}{4}$. 


\paragraph{Analysis.}
%
For the running example it is clear that reporting a non-zero SHAP
score for feature 2 and a zero-value SHAP score for feature 1 is not
only non-intuitive, but it also disagrees with the information
provided both by the adversarial examples and the abductive
explanations.
%
%
%
Motivated by these situations where exact SHAP scores are manifestely
unsatisfactory, there has been recent work on proposing alternatives
to SHAP scores~\cite{izza-aaai24,ignatiev-sat24}.
These recent alternatives to SHAP scores are based on power indices,
studied for assessing voting power. However, there exists a power
index that is based on Shapley values, namely the Shapley-Shubik
index, which is not studied in~\cite{izza-aaai24}.
This paper investigates how to change the definition of SHAP scores
such that: i) the unsatisfactory results of existing SHAP scores are
eliminated; and ii) an instantiation of the Shapley-Shubik index is
obtained.
The envisioned approach is to maintain the definition of SHAP scores
in terms of Shapley values, but change the characteristic function
that has been widely used for defining SHAP scores, i.e.\ $\cfn{e}$.%
\footnote{In the Appendix, we show that other well-known
alternatives~\cite{najmi-icml20} also yield unsatisfactory results.}
The next section investigates several novel characteristic functions,
all of which are based on the similarity predicate introduced earlier
in the paper.

\section{Properties of Characteristic Functions} \label{sec:pcfs}

Given the issues reported earlier in the paper, this section
proposes properties that characteristic functions should 
respect.
If characteristic functions fail to respect some of these properties,
then the resulting SHAP scores can provide misleading information
about relative feature importance.

\paragraph{Weak value independence.}
Let $\fml{M}_1=(\fml{F},\mbb{F},\fml{T}_1,\tau_1)$ be an ML model,
with domain $\mbb{D}_i$ for each feature $i\in\fml{F}$. Moreover, let
$\fml{M}_2=(\fml{F},\mbb{F},\fml{T}_2,\tau_2)$ be another ML model,
with the same domains, and with $|\fml{T}_1|=|\fml{T}_2|$.
In addition, let $\mu:\fml{T}_1\to\fml{T}_2$ be a surjective mapping
from $\fml{T}_1$ to $\fml{T}_2$, such that for any
$\mbf{x}\in\mbb{F}$, $\tau_2(\mbf{x})=\mu(\tau_1(\mbf{x}))$.
Finally, let the target samples be $(\mbf{v},q)$, for $\fml{M}_1$,
and $(\mbf{v},\mu(q))$ for $\fml{M}_2$, thus defining the explanation 
problems $\fml{E}_1=(\fml{M}_1,(\mbf{v},q))$ and
$\fml{E}_2=(\fml{M}_2,(\mbf{v},\mu(q)))$.
A characteristic function $\cfn{t}$ is \emph{weakly value-independent}
if, given surjective $\mu$,
$\forall(i\in\fml{F}).[\svn{t}(i;\fml{E}_1)=\svn{t}(i;\fml{E}_2)]$

\paragraph{Strong value independence.}
Let $\fml{M}_1=(\fml{F},\mbb{F},\fml{T}_1,\tau_1)$ be an ML model,
with domain $\mbb{D}_i$ for each feature $i\in\fml{F}$. 
Moreover, let $\fml{M}_2=(\fml{F},\mbb{F},\fml{T}_2,\tau_2)$ be
another classifier, with the same domains.
%
In addition, let $\mu:\fml{K}_1\to\fml{K}_2$ be a 
mapping from $\fml{T}_1$ to $\fml{T}_2$, such that for
$q\in\fml{T}_1$, 
and such that,
$\forall(b\in\fml{T}_1).[(b\not=q)\limply(\mu(b)\not=\mu(q))]$
%
%
Finally, let the target samples be $(\mbf{v},q)$, for $\fml{M}_1$,
and $(\mbf{v},\mu(q))$ for $\fml{M}_2$, thus defining the explanation
problems $\fml{E}_1=(\fml{M}_1,(\mbf{v},q))$ and
$\fml{E}_2=(\fml{M}_2,(\mbf{v},\mu(q)))$.
A characteristic function $\cfn{t}$ is \emph{strongly
value-independent} if, given $\mu$,
$\forall(i\in\fml{F}).[\svn{t}(i;\fml{E}_1)=\svn{t}(i;\fml{E}_2)]$
%
Given the above, the following result holds%
\footnote{%
All the proofs are included in the Appendix.
},

\begin{restatable}{rstprop}{PropStrongCIImpliesWeakCI}
  If a characteristic function is strongly value-independent, then it
  is weakly value-independent.
\end{restatable}

\paragraph{Compliance with feature (ir)relevancy.}
Characteristic functions should respect feature (ir)relevancy, i.e.\ a
feature is irrelevant iff its (corrected) SHAP score is 0. Formally,
a characteristic function $\cf_t$ is compliant with feature
(ir)relevancy if,
%
\begin{equation} \label{eq:compliance}
  \forall(i\in\fml{F}).\irrelevant(i)\leftrightarrow(\svn{t}(i)=0)
\end{equation}
In previous work~\cite{hms-corr23a,hms-ijar24,msh-cacm24}, SHAP scores
are said to be \emph{misleading} when compliance with feature
(ir)relevancy is not respected. In the remainder of the paper, we
assign the same meaning to the term \emph{misleading}.

\paragraph{Numerical neutrality.}
Existing definitions of SHAP scores are based on expected values and
so require $\fml{T}$ (i.e.\ $\fml{K}$ or $\mbb{C}$) to be ordinal.
However, classification problems often contemplate categorical
classes. A characteristic function respects numerical neutrality if it
can be used with both numerical and non-numerical $\mbb{T}$.

\paragraph{Discussion.}
The properties proposed in this section target the issues reported in
earlier work~\cite{msh-cacm24,hms-ijar24}, where SHAP scores mislead
with respect to relative feature importance. Additional properties
might be devised to address other hypothetical issues.

\section{New Characteristic Functions} \label{sec:ncfs}

This section proposes novel characteristic functions,%
\footnote{%
In the rest of the paper, the symbol $\cfn{t}$ will be used to denote
some concrete characteristic function distinguished by the letter
$t$. The SHAP scores obtained with such characteristic function
$\cfn{t}$ will be denoted by $\svn{T}$ (i.e.\ for scores we capitalize
the corresponding letter, and so the parameterization is
$\svn{T}(\fml{S};\fml{E},\cfn{t})$).
}
most of which respect all of the target properties identified in the
previous section.
%
For each characteristic function $\cfn{t}$, and given a fixed
explanation problem $\fml{E}$, the obtained SHAP scores will be
unique. Some of these SHAP scores will not respect the properties
proposed earlier in the paper, e.g.\ this is the case with $\cfn{e}$,
whereas some of the novel SHAP scores will respect all of those
properties, in addition to the axioms proved by
Shapley~\cite{shapley-ctg53}.
Throughout this section, an explanation problem
$\fml{E}=(\fml{M},(\mbf{v},q))$ is assumed, and it is used to
parameterize the proposed characteristic functions.

\paragraph{Defining the new characteristic functions.}
Given the definition of the similarity predicate, we now introduce the 
following main new characteristic functions.
\begin{align}
  %
  \cfn{s}(\fml{S};\fml{E}&) ~~ := ~~ 
  \exv[\similar(\mbf{x})\,|\,\mbf{x}_{\fml{S}}=\mbf{v}_{\fml{S}}]
  \label{eq:def:vs}
  \\[3pt]
  \cfn{a}(\fml{S};\fml{E}&) ~~ := ~~ 
  \left\{
  \begin{array}{lcl}
    1 & \quad & \tn{if
      $\cfn{s}(\fml{S};\fml{E})=1$%
    }\\[2pt]
    0 & \quad & \tn{otherwise}\\
  \end{array}
  \right.
  \label{eq:def:va}
  \\[3pt]
  %
  %
  \cfn{c}(\fml{S};\fml{E}&) ~~ := ~~ 
  \left\{
  \begin{array}{lcl}
    1 & \quad & \tn{if
      $\cfn{s}(\fml{F}\setminus\fml{S};\fml{E})<1$
    }\\[2pt]
    0 & \quad & \tn{otherwise}\\
  \end{array}
  \right.
  \label{eq:def:vc}
\end{align}
We will refer to characteristic functions $\cfn{e}$
(see~\eqref{eq:cfs}), $\cfn{s}$,
$\cfn{a}$, $\cfn{c}$, respectively as the \emph{expected value},
the \emph{similarity}, the WAXp-based, and the WCXp-based
characteristic functions.

Furthermore, we will introduce another characteristic function, which
is shown to be tightly related with $\cfn{a}$.
%
\begin{equation}  \label{eq:def:vna}
  \cfn{n}(\fml{S};\fml{E}) ~~ := ~~ 
  \left\{
  \begin{array}{lcl}
    1 & \quad & \tn{if
      $\cfn{s}(\fml{S};\fml{E})<1$%
    }\\[2pt]
    0 & \quad & \tn{otherwise}\\
  \end{array}
  \right.
\end{equation}
(Observe that $\cf_n$ can be viewed as the complement of $\cf_a$.)

It is plain that for each characteristic $\cfn{t}$, with
$t\in\{a,c,n,s\}$, one can create a corresponding Shapley value 
$\svn{T}$, with $T\in\{A,C,N,S\}$. It suffices to replace the
characteristic function $\cfn{e}$ used in~\eqref{eq:sv}, by $\cfn{t}$,
for the chosen $t\in\{a,c,n,s\}$.


\paragraph{Basic attributes of the new characteristic functions.}
We start by deriving some basic results regarding the characteristic
functions 
$\cfn{a}$, $\cfn{c}$ and $\cfn{n}$.
Throughout, it is assumed an explanation problem $\fml{E}$.


\begin{restatable}{rstprop}{CFaCFcCFnResOne}
  Given the definition of $\cfn{a}$, $\cfn{c}$ and $\cfn{n}$, then
  $\svn{A}(i;\fml{E},\cfn{a})\ge0$, $\svn{C}(i;\fml{E},\cfn{c})\ge0$,
  and $\svn{N}(i;\fml{E},\cfn{N})\le0$.
\end{restatable}

\begin{restatable}{rstprop}{CFaCFnResTwo}
  \label{prop:AvsNvsC}
  The following holds true:
  \begin{enumerate}[topsep=0pt,nosep]
  \item
    $\forall(\fml{S}\subseteq\fml{F}).[\cfn{a}(\fml{S};\fml{E})=1\leftrightarrow\waxp(\fml{S};\fml{E})]$.
  \item 
    $\forall(\fml{S}\subseteq\fml{F}).[\cfn{n}(\fml{S};\fml{E})=1\leftrightarrow\wcxp(\fml{F}\setminus\fml{S};\fml{E})]$.
  \item
    $\forall(\fml{S}\subseteq\fml{F}).[\cfn{c}(\fml{S};\fml{E})=1\leftrightarrow\wcxp(\fml{S};\fml{E})]$.
    %
  \end{enumerate}
\end{restatable}

\begin{restatable}{rstprop}{CFaCFcResOne}
  \label{prop:AeqCeqnN}
  The following holds true:
  \begin{enumerate}[nosep]
    %
  \item $\forall(i\in\fml{F}).[\svn{A}(i;\fml{E},\cfn{a}) = -\svn{N}(i;\fml{E},\cfn{n})]$.
  \item $\forall(i\in\fml{F}).[\svn{A}(i;\fml{E},\cfn{a})=\svn{C}(i;\fml{E},\cfn{c})]$.
  \end{enumerate}
\end{restatable}

An immediate consequence of the results in~\cref{prop:AeqCeqnN,prop:AvsNvsC},
is that the complexity of computing SHAP scores $\svn{T}$ is the same
for $T\in\{A,C,N\}$.

Furthermore, there are additional consequences
to~\cref{prop:AeqCeqnN}.
Observe that $\cfn{a}$ and $\cfn{c}$ are defined in terms of
predicates that are related by a hitting set relationship, i.e.\ as
noted earlier in the paper, each WAXp is a hitting set of the set of 
WCXps, and each WCXp is a hitting set of the set of WAXps. We call
such property \emph{hitting set duality}.

Thus, from~\cref{prop:AeqCeqnN}, we obtain a stronger result.
Consider two predicates $\pred_{\beta}$ and $\pred_{\delta}$, mapping
the powerset of $\fml{F}$, i.e. $2^{\fml{F}}$, to $\{\bot,\top\}$, and
such that $\pred_{\beta}$ and $\pred_{\delta}$ exhibit hitting set
duality. Define $\cfn{b}$ such that
$\cfn{b}(\fml{S})=1$ iff $\pred_{\beta}(\fml{S})$, and
$\cfn{d}(\fml{S})=1$ iff $\pred_{\delta}(\fml{S})$. From $\cfn{b}$ and
$\cfn{d}$, we obtain the Shapley values $\svn{B}$ and $\svn{D}$,
respectively.
As a result, 
\begin{theorem} \label{thm:dual}
  Given the definitions of $\cfn{b}$ and $\cfn{d}$, it is the case
  that: $\forall(i\in\fml{F}).\svn{B}(i)=\svn{D}(i)$.
\end{theorem}

It is interesting to observe that \cref{thm:dual} is a new result that
finds application beyond XAI, in any practical uses of Shapley values,
including those mentioned earlier in the paper.


\paragraph{Properties of the new characteristic functions.}
We now assess which of the properties of characteristic functions
proposed earlier 
are respected by which characteristic functions among those proposed
in this section.

It is plain that characteristic functions based on the similarity
predicate respect numerical neutrality.
Furthermore, another general result is that characteristic functions
based on the similarity predicate guarantee strong (and so weak) value
independence.

\begin{restatable}{rstprop}{StrongCIRes}
  For $t\in\{s,a,c,n\}$ and $i\in\fml{F}$, it is the case that the
  characteristic function $\cfn{t}$ 
  respects strong value independence.
\end{restatable}

\begin{restatable}{rstprop}{RelevancyComplianceRes} 
  \label{rstprop:RCR}
  For $T\in\{A,C,N\}$, then it is the case that
  $\forall(i\in\fml{F}).\irrelevant(i)\leftrightarrow(\svn{T}(i)=0)$.
\end{restatable}


Finally, we observe that $\cfn{s}$ represents a boolean classifier,
and so it  exhibits the issues with SHAP scores uncovered for boolean
classifiers based on $\cfn{e}$ in earlier
work~\cite{hms-corr23a,hms-ijar24}.


\section{Complexity of Computing SHAP Scores} \label{sec:cplsv}

The previous section introduced novel characteristic functions that
exhibit a number of desirable properties, which in turn ensure that
SHAP scores will not produce misleading information
(see~\cref{rstprop:RCR}).
%
%
Another related question is how the novel characteristic functions
impact the computional complexity of computing SHAP scores. This
section starts the effort of mapping such computional complexity.
%
%

\paragraph{Intractable cases.}
A number of intractability results have been obtained in recent
years~\cite{vandenbroeck-aaai21,vandenbroeck-jair22}.
As noted earlier in the paper, for boolean functions, the similarity
predicate does not provide any difference with respect to the original
classifier. The following result is clear.

\begin{restatable}{rstprop}{BoolClfZeta}
  \label{eqzk}
  For a boolean classifier, with $\kappa(\mbf{v})=1$, then
  $\forall(\mbf{x}\in\mbb{F}).\left(\similar(\mbf{x};\fml{E})=\kappa(\mbf{x})\right)$.
\end{restatable}

From~\cref{eqzk} and Corollary~8 in~\cite{vandenbroeck-jair22}, it is
immediate that,

\begin{restatable}{rstprop}{CplCFsCnfDnfZeta}
  \label{prop:cnfdnf}
  Computing SHAP scores $\svn{S}$ is \#P-hard for boolean classifiers
  in CNF or DNF.
\end{restatable}

Clearly, given~\cref{prop:cnfdnf}, then the computation of SHAP
scores for more complex boolean classifiers is also \#P-hard.

Moreover, a key recent result regarding the computation of SHAP scores
is that for the characteristic function $\cfn{e}$ there are
polynomial-time algorithms for computing
$\svn{E}$~\cite{barcelo-aaai21,barcelo-jmlr23}.
In contrast, for characteristic functions that build on WAXps/WCXps,
the computation of SHAP scores becomes NP-hard, even for d-DNNF and
DDBC classifiers.%
  \footnote{Deterministic Decomposable Boolean Circuits (DDBCs) and
  deterministic Decomposable Negation Normal (d-DNNF) form circuits
  denote well-known restrictions of boolean circuits, and are briefly
  overviewed in the supplemental materials.} 

\begin{restatable}{rstprop}{ScddNNfNpHard}
  \label{prop:ddnnf}
  For $T\in\{A,C,N\}$, the computation of the SHAP scores $\svn{T}$ is
  NP-hard for d-DNNF \& DDBC classifiers.
\end{restatable}

%





\paragraph{Polynomial-time cases.}
As shown above, the most significant tractability result that is known
for $\cfn{e}$ does not hold for $\cfn{t}$, with $t\in\{a,c,n\}$.
Nevertheless, some tractability results can be proved.
%
%
%
%
For ML models represented by tabular representations (e.g.\ truth
tables), it is simple to devise algorithms polynomial on the size of
the classifier's representation~\cite{hms-corr23a}. As a result, it is
the case that,
\begin{restatable}{rstprop}{ScPolyTabRep}
  \label{prop:poly:tr}
  There exist polynomial-time algorithms for computing the SHAP scores
  $\svn{S}$, $\svn{A}$, $\svn{C}$ for ML models represented by
  tabular representations.
\end{restatable}

Since the recent results on the tractability of computing SHAP scores
for deterministic and decomposable circuits
(d-DNNFs)~\cite{barcelo-aaai21,barcelo-jmlr23} considering boolean
classifiers, then from~\cref{eqzk} and~\cite{barcelo-jmlr23}, it is
the case that,
\begin{restatable}{rstprop}{ScPolyddNNFSim}
  \label{prop:poly:ddnnf}
  The computation of SHAP scores $\svn{S}$ is in P for classifiers
  represented by non-boolean DDBCs. 
\end{restatable}

\section{Similarity-Based SHAP}
\label{sec:sshap}

This section outlines a first step towards addressing the issues with
SHAP scores reported in earlier work, and observed in the tool
SHAP~\cite{lundberg-nips17}.
Instead of running SHAP with the original training data and the
original classifier, the similarity-based SHAP (referred to as sSHAP)
replaces the original classifier by the similarity predicate, and
reorganizes training data accordingly. In terms of running time
complexity, the impact of the modifications to SHAP are negligible.
More importantly, sSHAP will be approximating $\svn{S}$, since the
underlying characteristic funtion is $\cfn{s}$.
In practice, sSHAP is built on top of the SHAP
tool~\cite{lundberg-nips17}.
%

As noted earlier in the paper, the use of $\cfn{s}$ does not
guarantee the non-existence of some of the issues reported in earlier 
work~\cite{hms-ijar24}, since it is known that even boolean
classifiers can exhibit a number of issues related with the relative
order of feature importance. 
Nevertheless, another question is whether $\cfn{s}$ can serve to
correct SHAP scores (obtained with $\cfn{e}$) in classifiers for which
the reported issues rely on non-boolean classification.

\paragraph{Difference in SHAP scores for example classifiers.}
To validate the improvements obtained with $\cfn{s}$ with respect to
$\cfn{e}$, we studied the non-boolean classifiers reported
in~\cite{hms-ijar24}%
\footnote{%
From~\cite{hms-ijar24}, we consider
(i) the two DTs of case study 2 (Fig.~3 in~\cite{hms-ijar24}),
referred to as cs02a and cs02b;
(ii) the two DTs of case study 3 (Fig.~5 in~\cite{hms-ijar24}),
referred to as cs03a and cs03b; and
(iii) the two DTs of case study 4 (Fig.~8 in~\cite{hms-ijar24}),
referred to as cs04a and cs04b.
Moreover, for cs02a, cs02b, cs03a and cs03b there exist 16 instances,
whereas for cs04a and cs04b there exist 24 instances (because of a
discrete but non-boolean domain for one of the features.)
}.
For each classifier, each of the possible instances is analyzed, and
the SHAP scores produced by the tools SHAP and sSHAP are recorded. If
an irrelevant feature is assigned an absolute value larger than some
other relevant feature, then a mismatch is declared.
\cref{tab:shap_sshap_dt} summarizes the results obtained with the two
tools, where columns \emph{SHAP-FRP mismatch} shown the number of
mismatches obtained with SHAP, and column \emph{sSHAP-FRP mismatch}
shows the number of mismatches obtained with sSHAP\footnote{%
It should be noted that sSHAP can replace SHAP in any application
domain, ensuring similar performance. However, a more extensive
assessment of the quality of the results of the two tools is
unrealistic at present; we would have to be able to compute exact SHAP
scores, and this is only computationally feasible for very simple ML
models, e.g.\ restricted examples of DTs.}.
\begin{table}[t]
  \centering
  \renewcommand{\tabcolsep}{0.485em}
  \begin{tabular}{ccc}
    \toprule
    DT & SHAP-FRP mismatch & sSHAP-FRP mismatch \\ 
    \midrule
    cs02a & 11 & 0 \\ 
    cs02b & 4 & 0 \\ 
    cs03a & 5 & 0 \\ 
    cs03b & 4 & 0 \\ 
    cs04a & 15 & 0 \\ 
    cs04b & 4 & 0 \\ 
    \bottomrule
  \end{tabular}
  \caption{Comparing SHAP with sSHAP.}
  \label{tab:shap_sshap_dt}
\end{table}
As can be concluded, SHAP produces several mismatches. In contrast,
sSHAP produces no mismatch. It should be noted that both tools are
approximating the SHAP scores given the respective characteristic
functions, i.e. the computed scores are not necessarily the ones
dictated by \eqref{eq:sv}.

Observe that $\cfn{s}$ consists of replacing the original classifier
by a new boolean classifier. Hence, from~\cite{hms-ijar24},
such boolean classifiers can also produce misleading information.
Nevertheless, given the results above and other experiments,
in the cases where $\cfn{s}$ was used, we were unable to observe 
some of the issues proposed in earlier
work~\cite{msh-cacm24,hms-ijar24}. It is the subject of future work to
decide whether such issues can occur for boolean classifiers.

\section{Conclusions} \label{sec:conc}

Recent work demonstrated the existence of classifiers for which the
exact SHAP scores are unsatisfactory.
%
%
This paper argues that the issues identified with SHAP scores result
from the characteristic functions used in earlier work. As a result,
the paper devises several properties which characteristic functions
must respect in order to compute SHAP scores that do not exhibit those
issues.
Complexity-wise, the paper argues that the proposed characteristic
functions are as hard to compute as the characteristic functions used
in earlier works studying the complexity of SHAP
scores~\cite{vandenbroeck-aaai21,barcelo-aaai21,vandenbroeck-jair22,barcelo-jmlr23},
or harder. 
%
Finally, the paper proposes simple modifications to the tool
SHAP~\cite{lundberg-nips17}, thereby obtaining SHAP scores that
respect some of the proposed properties.
%

%

\appendix

\section{Appendix} \label{app:appendix}



%
\subsection{Proofs}
\label{app:proofs}

\PropStrongCIImpliesWeakCI*

\begin{proof}
  If a characteristic function is strongly value independent, it
  suffices to restrict the choices of $\mu$ to surjective functions to
  make it weakly value independent.
\end{proof}


\CFaCFcCFnResOne*

\begin{proof}(Sketch)
  We only consider $\cfn{a}$. (The proof for $\cfn{c}$ and $\cfn{n}$
  follows from~\cref{prop:AeqCeqnN}.)\\
  It is plain that $\Delta_{i}(\fml{S};\fml{E},\cfn{a})\in\{-1,0,1\}$,
  given the possible values that $\cfn{a}$ can take.
  In fact, it is the case that
  $\Delta_{i}(\fml{S};\fml{E},\cfn{a})\in\{0,1\}$.
  If a set $\fml{S}\subseteq\fml{F}$ is a WAXp, then a proper superset
  is also a WAXp; hence it is never the case that
  $\Delta_{i}(\fml{S};\fml{E},\cfn{a})=-1$.
  Since every $\Delta_{i}(\fml{S};\fml{E},\cfn{a})\ge0$, then
  $\svn{A}(i;\fml{E},\cfn{a})\ge0$.
\end{proof}

\CFaCFnResTwo*
%

\begin{proof}
  We consider each case separately:
  \begin{enumerate}[nosep]
  \item If $\cfn{a}(\fml{S};\fml{E})=1$, then, as noted earlier in the
    paper, $\similar(\mbf{x};\fml{E})=1$ for all points
    $\mbf{x}\in\Upsilon(\fml{S})$, and so the ML model's prediction
    is indistinguishable from $q$ for all points in
    $\Upsilon(\fml{S})$. Hence, by definition, $\fml{S}$ is a WAXp.  
    Conversely, if $\fml{S}$ is an WAXp, then the prediction must be
    indistinguishable from $q$ for all points $\mbf{x}$ in
    $\Upsilon(\fml{S})$,
    $\forall(\mbf{x}\in\Upsilon(\fml{S})).[\similar(\mbf{x};\fml{E})]$. Thus,
    $\cfn{a}(\fml{S};\fml{E})=1$.
  \item If $\cfn{n}(\fml{S};\fml{E})=1$, then 
    $\similar(\mbf{x};\fml{E})\not=1$ for some point(s)
    $\mbf{x}\in\Upsilon(\fml{S})$, and so the ML model's prediction
    is distinguishable from $q$ for some point(s) in
    $\Upsilon(\fml{S})$. Hence, by definition,
    $\fml{F}\setminus\fml{S}$ is a WCXp. 
    Conversely, if $\fml{F}\setminus\fml{S}$ is an WCXp, then the
    prediction must be distinguishable from $q$ for some point(s)
    $\mbf{x}$ in $\Upsilon(\fml{S})$,
    i.e.\ $\exists(\mbf{x}\in\Upsilon(\fml{S})).[\similar(\mbf{x};\fml{E})\not=1]$. Thus,
    $\cfn{n}(\fml{S};\fml{E})=1$. 
  \item If $\cfn{c}(\fml{S};\fml{E})=1$, then
    $\similar(\mbf{x};\fml{E})\not=1$ for some point(s)
    $\mbf{x}\in\Upsilon(\fml{F}\setminus\fml{S})$, and so the 
    ML model's prediction is distinguishable from $q$ for some
    point(s) in $\Upsilon(\fml{F}\setminus\fml{S})$. Hence, by
    definition, $\fml{S}$ is a WCXp. Conversely, if $\fml{S}$ is an
    WCXp, then the prediction must be distinguishable from $q$ for
    some point(s) $\mbf{x}$ in $\Upsilon(\fml{F}\setminus\fml{S})$,
    i.e.\ $\exists(\mbf{x}\in\Upsilon(\fml{F}\setminus\fml{S})).[\similar(\mbf{x};\fml{E})\not=1]$. Thus,
    $\cfn{c}(\fml{S};\fml{E})=1$. 
    %
    \qedhere
  \end{enumerate}
\end{proof}

\CFaCFcResOne*


\begin{proof}
  We consider each case separately:
  \begin{enumerate}[nosep]
    \item By definition, it is plain that
    $\cfn{a}(\fml{S};\fml{E})+\cfn{n}(\fml{S};\fml{E})=1$, for any
    $\fml{S}\subseteq\fml{F}$, because it must be the case that 
    either $\cfn{s}(\fml{S};\fml{E})=1$ or
    $\cfn{s}(\fml{S};\fml{E})<1$, but not both.
    Given the values that $\cfn{a}(\fml{S};\fml{E})$ can take, it is
    also plain that
    $\Delta_{i}(\fml{S};\fml{E},\cfn{a})\in\{\tm1,0,1\}$. Moreover, if 
    $\Delta_{i}(\fml{S};\fml{E},\cfn{a})=\tm1$, then
    $\Delta_{i}(\fml{S};\fml{E},\cfn{n})=1$. If
    $\Delta_{i}(\fml{S};\fml{E},\cfn{a})=1$, then
    $\Delta_{i}(\fml{S};\fml{E},\cfn{n})=\tm1$. Also, if
    $\Delta_{i}(\fml{S};\fml{E},\cfn{a})=0$, then
    $\Delta_{i}(\fml{S};\fml{E},\cfn{n})=0$.
    Thus, for any $i\in\fml{F}$ and $\fml{S}\subseteq\fml{F}$,
    $\Delta_{i}(\fml{S};\fml{E},\cfn{n})=\tm\Delta_{i}(\fml{S};\fml{E},\cfn{a})$.
    Hence, the result follows.
  \item
    Since
    $\forall(\fml{S}\subseteq\fml{F}).\wcxp(\fml{F}\setminus\fml{S};\fml{E})\leftrightarrow\neg\waxp(\fml{S};\fml{E})$,
    by definition, then we have
    $\forall(i\in\fml{F}),\forall(\fml{S}\subseteq(\fml{F}\setminus\{i\}))$,
    \begin{align}
      & \Delta_{i}(\fml{S};\fml{E},\cfn{a})=1
      \nonumber\\
      \Leftrightarrow&
      \neg\waxp(\fml{S};\fml{E})\land\waxp(\fml{S}\cup\{i\};\fml{E})
      \nonumber\\
      \Leftrightarrow&
      \wcxp(\fml{F}\setminus\fml{S};\fml{E})\land\neg\wcxp(\fml{F}\setminus(\fml{S}\cup\{i\});\fml{E})
      \nonumber\\
      \Leftrightarrow&
      \wcxp(\fml{F}\setminus\fml{S};\fml{E})\land\neg\wcxp((\fml{F}\setminus\{i\})\setminus\fml{S};\fml{E})
      \nonumber\\
      \Leftrightarrow&
      \Delta_{i}((\fml{F}\setminus\{i\})\setminus\fml{S};\fml{E},\cfn{c})=1
      \nonumber
    \end{align}
    Now, let
    $\Phi(i):=\{\fml{S}\subseteq(\fml{F}\setminus\{i\})\,|\,\Delta_{i}(\fml{S};\fml{E},\cfn{a})=1\}$. Then,
    by construction,
    $\svn{A}(i;\fml{E},\cfn{a})=\sum_{\fml{S}\in\Phi(i)}\varsigma(|\fml{S}|)$
    (because $\Delta_{i}(\fml{S};\fml{E},\cfn{a})=0$ otherwise) and,
    by the equivalence above,
    $\svn{C}(i;\fml{E},\cfn{c})=\sum_{\fml{S}\in\Phi(i)}\varsigma(|\fml{F}\setminus\{i\})\setminus\fml{S}|)$.
    However, it is immediate to prove that 
    $\varsigma(|\fml{S}|)=\varsigma(|\fml{F}\setminus\{i\})\setminus\fml{S}|)$,
    and so the two sums are also equal. 
    This proves the result.
    \qedhere
  \end{enumerate}
\end{proof}



%

\StrongCIRes*

\begin{proof}
  For a characteristic function to respect strong value independence,
  the SHAP scores must not change if the values are mapped using some
  function $\mu$. 
  By hypothesis, for any point
  $\mbf{x}\in\mbb{F}$, the resulting ML models will predict $\mu(q)$
  iff the original ML model predicts $q$. This means the resulting
  similarity predicates are the same for the two ML models, and so
  the SHAP scores $\svn{T}$, $T\in\{S,A,C,N\}$, remain unchanged.
\end{proof}

\RelevancyComplianceRes*

\begin{proof}
  First, we consider $\cfn{a}$. 
  Let $i\in\fml{F}$ be an irrelevant feature. 
  It is plain that $\Delta_{i}(\fml{S};\fml{E},\cfn{a})\in\{-1,0,1\}$,
  given the possible values that $\cfn{a}$ can take. However, as
  argued above, $\Delta_{i}(\fml{S};\fml{E},\cfn{a})\in\{0,1\}$, since
  if a set $\fml{S}\subseteq\fml{F}$ is a WAXp, then a proper superset
  is also a WAXp; hence it is never the case that
  $\Delta_{i}(\fml{S};\fml{E},\cfn{a})=-1$. 
  We are interested in the sets
  $\fml{S}\subseteq(\fml{F}\setminus\{i\})$ for which
  $\Delta_{i}(\fml{S};\fml{E},\cfn{a})=1$, since these are the only
  ones that contribute to making $\svn{A}(i;\fml{E},\cfn{a})\not=0$.
  For $\Delta_{i}(\fml{S};\fml{E},\cfn{a})=1$, it must be the case
  that $\cfn{a}(\fml{S};\fml{E})=0$ and
  $\cfn{a}(\fml{S}\cup\{i\};\fml{E})=1$.
  However, this would imply that $i$ would be included in some
  AXp~\cite{hcmpms-tacas23}.
  But $i$ is irrelevant, and so it is not
  included in any AXp. Hence, there exists no set
  $\fml{S}\subseteq(\fml{F}\setminus\{i\})$ such that
  $\Delta_{i}(\fml{S};\fml{E},\cfn{a})=1$, and so
  $\svn{A}(i;\fml{E},\cfn{a})=0$.\\ 
  Let $\svn{A}(i;\fml{E},\cfn{a})=0$. 
  An analysis similar to the above one allows concluding that there
  exist no sets $\fml{S}$ such that
  $\Delta_{i}(\fml{S};\fml{E},\cfn{a})=1$. Hence, it is never the case
  that $\cfn{a}(\fml{S};\fml{E})=0$ and
  $\cfn{a}(\fml{S}\cup\{i\};\fml{E})=1$. Thus, $i$ is not included in
  any AXp, and so it is irrelevant.\\
  For $\cfn{c}$ and $\cfn{n}$, it suffices to
  invoke~\cref{prop:AeqCeqnN}; hence, the features for which
  $\svn{A}(i)=0$ are exactly the ones for which $\svn{C}(i)=0$ and
  $\svn{N}(i)=0$.\\
  This concludes the proof that $\svn{T}$, with $T\in\{A,C,N\}$,
  respects \eqref{eq:compliance}.
\end{proof}

%

\CplCFsCnfDnfZeta*

\begin{proof}
  From~\cite{vandenbroeck-aaai21,vandenbroeck-jair22}, it is known
  that computing SHAP scores is polynomially equivalent to computing
  the expected value. In the boolean case, and so in the case of
  $\cfn{s}$, this is polynomially equivalent to model counting.
  Furthermore, model counting for DNF and CNF formulas is
  \#P-complete~\cite{valiant-sjc79}. Thus, computing the SHAP scores
  using $\cfn{s}$ is \#P-hard.
\end{proof}

\ScddNNfNpHard*

\begin{proof}
  We reduce the problem of feature relevancy to the problem of
  computing the SHAP scores $\svn{T}$, with $T\in\{A,C,N\}$.
  Since feature relevancy is NP-complete for d-DNNF
  circuits~\cite{hcmpms-tacas23}, this proves that computing the SHAP
  scores $\svn{T}$, with $T\in\{A,C,N\}$ is NP-hard.\\
  Given an explanation problem we can decide feature membership as
  follows. We compute the SHAP score for each feature $i\in\fml{F}$.
  Moreover, since $\cfn{t}$, $t\in\{a,c,n\}$ are compliant with
  feature (ir)relevancy (see~\eqref{eq:compliance}), then
  $\svn{T}(i;\fml{E},\cfn{t})=0$ iff feature $i$ is irrelevant. Hence,
  if we could compute the SHAP scores in polynomial-time, then we
  could decide feature relevancy in polynomial-time, and so computing
  the SHAP-scores for d-DNNFs is NP-hard.\\
  Now, since DDBCs generalize d-DNNFs~\cite{barcelo-jmlr23}, then 
  computing the SHAP-scores for DDBCs is also NP-hard.
\end{proof}

%

%


%
\subsection{Limitations of SHAP Scores Based on Baselines}
\label{app:blines}

We focus on BShap~\cite{najmi-icml20}; similar analyzes could be made
for other baselines~\cite{blockbaum-aistats20,najmi-icml20}.

Throughout this section, the baseline is a point $\mbf{w}\in\mbb{F}$. 
Furthermore, for each $\fml{S}\subseteq\fml{F}$, let
$\mbf{x}^{\fml{S}}_{b}$ be such that
$x^{\fml{S}}_{b,i}=\tn{ITE}(i\in\fml{S},v_i,w_i)$.

Given $\mbf{w}\in\mbb{F}$, the BShap characteristic function $\cfn{b}$
is defined by $\cfn{b}(\fml{S})=\kappa(\mbf{x}^{\fml{S}}_b)$, for
$\fml{S}\subseteq\fml{F}$.

\paragraph{Remarks about baselines.}
Analysis of the definition of BShap~\cite{najmi-icml20} allows proving
the following results.

\begin{proposition}
  The following holds:
  \begin{enumerate}[nosep]
  \item BShap is only well-defined if all the domains are boolean,
    i.e. $\mbb{F}=\{0,1\}^m$.
  \item BShap is only well-defined when $\mbf{w}=\neg\mbf{v}$.
  \end{enumerate}
\end{proposition}

\begin{proof}
  By contradiction, let us consider $i\in\fml{F}$, such that either
  $|\mbb{D}_i|>2$ or $w_i=v_i$. Then there exists a point
  $\mbf{z}\in\mathbb{F}$ such that $z_i\not\in\{v_i,w_i\}$.
  By construction, for each $\fml{S}\subseteq\fml{F}$,
  $\mbf{x}^{\fml{S}}_b$ is different from $\mbf{z}$. Thus, $\cfn{b}$
  and so $\svn{b}$ do not depend on $\kappa(\mbf{z})$. Therefore, we
  can use the value of $\kappa(\mbf{z})$ to change the AXps (and CXps)
  without modifying the BShap scores. As there are at least
  $2^{(|\fml{F}|-1)}$ such points $\mbf{z}$, it is plain that
  constructing counterexample is simple.
\end{proof}

\paragraph{BShap also misleads.}
%
The following notation is used $\fml{S}\subseteq\fml{F}$, let
$\mbf{v}^{\fml{S}}$ be defined by
$v^{\fml{S}}_{i}=\tn{ITE}(i\in\fml{S},v_i,\neg{v_i})$, with
$i\in\fml{F}$.

For $\fml{S}\subseteq\fml{F}$, then
$\cfn{b}(\fml{S})=\kappa(\mbf{v}^{\fml{S}})$.

\begin{proposition}
  $\cfn{b}$ misleads.
\end{proposition}

\begin{proof}
  Let $\kappa(x_1,x_2)=\tn{ITE}(x_1=1,1,2x_2)$, and instance
  $(\mbf{v},c)=((1,1),1)$.\\
  It is plain that feature 1 influences both selecting the prediction
  1 and changing the prediction to some other value. In contrast,
  feature 2 has not influence in either setting or changing the
  prediction of class 1.\\
  It is also plain that the set of AXps is $\{\{1\}\}$, and also that
  $\kappa(x_1,x_2)=1$ iff $x_1=1$.\\
  However, if we compute $\svn{b}$, we get $\svn{b}(1)=0$ and
  $\svn{b}(2)=1$, which is of course misleading.\\
  To confirm the SHAP scores, we proceed as follows.
  $\cfn{b}(\emptyset)=\kappa(0,0)=0$,
  $\cfn{b}(\{1\})=\kappa(1,0)=1$,
  $\cfn{b}(\{2\})=\kappa(0,1)=2$, and
  $\cfn{b}(\{1,2\})=\kappa(1,1)=1$.\\
  Thus,
  $\Delta_{b}(1,\emptyset)=\cfn{b}(\{1\})-\cfn{b}(\emptyset)=1$,
  $\Delta_{b}(1,\{2\})=\cfn{b}(\{1,2\})-\cfn{b}(\{2\})=-1$,
  $\Delta_{b}(2,\emptyset)=\cfn{b}(\{2\})-\cfn{b}(\emptyset)=2$,
  $\Delta_{b}(2,\{1\})=\cfn{b}(\{1,2\})-\cfn{b}(\{2\})=0$.\\
  And finally,
  $\svn{b}(1)=\sfrac{(\Delta_{b}(1,\{2\})+\Delta_{b}(1,\emptyset))}{2}=0$,
  $\svn{b}(2)=\sfrac{(\Delta_{b}(2,\{1\})+\Delta_{b}(2,\emptyset))}{2}=1$.
  \qedhere
\end{proof}



%
\section*{Acknowledgements}



We thank the anonymous reviewers for the helpful comments.
This work was supported in part by the Spanish Government under
grant PID2023-152814OB-I00, and by ICREA starting funds.
This work was supported in part by the National Research Foundation,
Prime Minister’s Office, Singapore under its Campus for Research
Excellence and Technological Enterprise (CREATE) programme.
This work was also supported in part by the AI Interdisciplinary
Institute ANITI, funded by the French program ``Investing for the
Future -- PIA3'' under Grant agreement no.\ ANR-19-PI3A-0004.

\newtoggle{mkbbl}

\settoggle{mkbbl}{false}

\iftoggle{mkbbl}{
    \bibliography{refs}

\begin{thebibliography}{58}
\providecommand{\natexlab}[1]{#1}

\bibitem[{Afchar, Guigue, and Hennequin(2021)}]{guigue-icml21}
Afchar, D.; Guigue, V.; and Hennequin, R. 2021.
\newblock Towards Rigorous Interpretations: a Formalisation of Feature
  Attribution.
\newblock In \emph{ICML}, 76--86.

\bibitem[{Amgoud, Ben{-}Naim, and Vesic(2017)}]{amgoud-ijcai17}
Amgoud, L.; Ben{-}Naim, J.; and Vesic, S. 2017.
\newblock Measuring the Intensity of Attacks in Argumentation Graphs with
  {S}hapley Value.
\newblock In \emph{IJCAI}, 63--69.

\bibitem[{Arenas et~al.(2021)Arenas, Barcel{\'{o}}, Bertossi, and
  Monet}]{barcelo-aaai21}
Arenas, M.; Barcel{\'{o}}, P.; Bertossi, L.~E.; and Monet, M. 2021.
\newblock The Tractability of {SHAP}-Score-Based Explanations for
  Classification over Deterministic and Decomposable Boolean Circuits.
\newblock In \emph{AAAI}, 6670--6678.

\bibitem[{Arenas et~al.(2023)Arenas, Barcel{\'{o}}, Bertossi, and
  Monet}]{barcelo-jmlr23}
Arenas, M.; Barcel{\'{o}}, P.; Bertossi, L.~E.; and Monet, M. 2023.
\newblock On the Complexity of {SHAP}-Score-Based Explanations: Tractability
  via Knowledge Compilation and Non-Approximability Results.
\newblock \emph{J. Mach. Learn. Res.}, 24: 63:1--63:58.

\bibitem[{Biradar et~al.(2023)Biradar, Izza, Lobo, Viswanathan, and
  Zick}]{izza-corr23}
Biradar, G.; Izza, Y.; Lobo, E.; Viswanathan, V.; and Zick, Y. 2023.
\newblock Axiomatic Aggregations of Abductive Explanations.
\newblock \emph{CoRR}, abs/2310.03131.

\bibitem[{Biradar et~al.(2024)Biradar, Izza, Lobo, Viswanathan, and
  Zick}]{izza-aaai24}
Biradar, G.; Izza, Y.; Lobo, E.; Viswanathan, V.; and Zick, Y. 2024.
\newblock Axiomatic Aggregations of Abductive Explanations.
\newblock In \emph{AAAI}, 11096--11104.

\bibitem[{Breiman et~al.(1984)Breiman, Friedman, Olshen, and
  Stone}]{breiman-bk84}
Breiman, L.; Friedman, J.; Olshen, R.; and Stone, C. 1984.
\newblock \emph{Classification and regression trees}.
\newblock Chapman and Hall.
\newblock ISBN 9781315139470.

\bibitem[{Brix et~al.(2023)Brix, M{\"{u}}ller, Bak, Johnson, and
  Liu}]{johnson-sttt23}
Brix, C.; M{\"{u}}ller, M.~N.; Bak, S.; Johnson, T.~T.; and Liu, C. 2023.
\newblock First three years of the international verification of neural
  networks competition {(VNN-COMP)}.
\newblock \emph{Int. J. Softw. Tools Technol. Transf.}, 25(3): 329--339.

\bibitem[{Campbell et~al.(2022)Campbell, Roder, Georgantas~III, and
  Roder}]{roder-mlwa22}
Campbell, T.~W.; Roder, H.; Georgantas~III, R.~W.; and Roder, J. 2022.
\newblock Exact {S}hapley values for local and model-true explanations of
  decision tree ensembles.
\newblock \emph{Machine Learning with Applications}, 9: 100345.

\bibitem[{Chalkiadakis, Elkind, and Wooldridge(2012)}]{elkind-bk12}
Chalkiadakis, G.; Elkind, E.; and Wooldridge, M.~J. 2012.
\newblock \emph{Computational Aspects of Cooperative Game Theory}.
\newblock Synthesis Lectures on Artificial Intelligence and Machine Learning.
  Morgan {\&} Claypool Publishers.
\newblock ISBN 978-3-031-00430-8.

\bibitem[{Darwiche(2023)}]{darwiche-lics23}
Darwiche, A. 2023.
\newblock Logic for Explainable {AI}.
\newblock In \emph{{LICS}}, 1--11.

\bibitem[{Darwiche and Marquis(2002)}]{darwiche-jair02}
Darwiche, A.; and Marquis, P. 2002.
\newblock A Knowledge Compilation Map.
\newblock \emph{J. Artif. Intell. Res.}, 17: 229--264.

\bibitem[{Dubey(1975)}]{dubey-ijgt75}
Dubey, P. 1975.
\newblock On the uniqueness of the {S}hapley value.
\newblock \emph{International Journal of Game Theory}, 4: 131--139.

\bibitem[{Felsenthal and Machover(1998)}]{machover-bk98}
Felsenthal, D.~S.; and Machover, M. 1998.
\newblock \emph{The measurement of voting power}.
\newblock Edward Elgar Publishing.

\bibitem[{Felsenthal and Machover(2015)}]{machover-hscv15}
Felsenthal, D.~S.; and Machover, M. 2015.
\newblock The measurement of a priori voting power.
\newblock In Heckelman, J.~C.; and Miller, N.~R., eds., \emph{Handbook of
  Social Choice and Voting}, chapter~08, 117--139. Edward Elgar Publishing.

\bibitem[{Fryer, Str{\"{u}}mke, and Nguyen(2021)}]{nguyen-ieee-access21}
Fryer, D.~V.; Str{\"{u}}mke, I.; and Nguyen, H.~D. 2021.
\newblock Shapley Values for Feature Selection: The Good, the Bad, and the
  Axioms.
\newblock \emph{{IEEE} Access}, 9: 144352--144360.

\bibitem[{Goodfellow, Shlens, and Szegedy(2015)}]{szegedy-iclr15}
Goodfellow, I.~J.; Shlens, J.; and Szegedy, C. 2015.
\newblock Explaining and Harnessing Adversarial Examples.
\newblock In \emph{ICLR}.

\bibitem[{Horn and Johnson(2012)}]{horn-bk12}
Horn, R.~A.; and Johnson, C.~R. 2012.
\newblock \emph{Matrix analysis}.
\newblock Cambridge university press.

\bibitem[{Huang et~al.(2023)Huang, Cooper, Morgado, Planes, and
  Marques{-}Silva}]{hcmpms-tacas23}
Huang, X.; Cooper, M.~C.; Morgado, A.; Planes, J.; and Marques{-}Silva, J.
  2023.
\newblock Feature Necessity {\&} Relevancy in {ML} Classifier Explanations.
\newblock In \emph{TACAS}, 167--186.

\bibitem[{Huang and Marques{-}Silva(2023)}]{hms-corr23a}
Huang, X.; and Marques{-}Silva, J. 2023.
\newblock The Inadequacy of {S}hapley Values for Explainability.
\newblock \emph{CoRR}, abs/2302.08160.

\bibitem[{Huang and Marques-Silva(2024)}]{hms-ijar24}
Huang, X.; and Marques-Silva, J. 2024.
\newblock On the failings of {S}hapley values for explainability.
\newblock \emph{International Journal of Approximate Reasoning}, 109112.

\bibitem[{Hunter and Konieczny(2006)}]{hunter-kr06}
Hunter, A.; and Konieczny, S. 2006.
\newblock Shapley Inconsistency Values.
\newblock In \emph{KR}, 249--259.

\bibitem[{Hunter and Konieczny(2010)}]{hunter-aij10}
Hunter, A.; and Konieczny, S. 2010.
\newblock On the measure of conflicts: Shapley Inconsistency Values.
\newblock \emph{Artif. Intell.}, 174(14): 1007--1026.

\bibitem[{Ignatiev, Narodytska, and
  Marques{-}Silva(2019{\natexlab{a}})}]{inms-aaai19}
Ignatiev, A.; Narodytska, N.; and Marques{-}Silva, J. 2019{\natexlab{a}}.
\newblock Abduction-Based Explanations for Machine Learning Models.
\newblock In \emph{AAAI}, 1511--1519.

\bibitem[{Ignatiev, Narodytska, and
  Marques{-}Silva(2019{\natexlab{b}})}]{inms-nips19}
Ignatiev, A.; Narodytska, N.; and Marques{-}Silva, J. 2019{\natexlab{b}}.
\newblock On Relating Explanations and Adversarial Examples.
\newblock In \emph{NeurIPS}, 15857--15867.

\bibitem[{James et~al.(2017)James, Witten, Hastie, and Tibshirani}]{james-bk17}
James, G.; Witten, D.; Hastie, T.; and Tibshirani, R. 2017.
\newblock \emph{An introduction to statistical learning}.
\newblock Springer.

\bibitem[{Janzing, Minorics, and Bl{\"{o}}baum(2020)}]{blockbaum-aistats20}
Janzing, D.; Minorics, L.; and Bl{\"{o}}baum, P. 2020.
\newblock Feature relevance quantification in explainable {AI:} {A} causal
  problem.
\newblock In \emph{AISTATS}, 2907--2916.

\bibitem[{Kumar et~al.(2021)Kumar, Scheidegger, Venkatasubramanian, and
  Friedler}]{friedler-nips21}
Kumar, I.; Scheidegger, C.; Venkatasubramanian, S.; and Friedler, S.~A. 2021.
\newblock Shapley Residuals: Quantifying the limits of the {S}hapley value for
  explanations.
\newblock In \emph{NeurIPS}, 26598--26608.

\bibitem[{Kumar et~al.(2020)Kumar, Venkatasubramanian, Scheidegger, and
  Friedler}]{friedler-icml20}
Kumar, I.~E.; Venkatasubramanian, S.; Scheidegger, C.; and Friedler, S.~A.
  2020.
\newblock Problems with {S}hapley-value-based explanations as feature
  importance measures.
\newblock In \emph{ICML}, 5491--5500.

\bibitem[{Letoffe, Huang, and Marques{-}Silva(2024)}]{lhms-corr24}
Letoffe, O.; Huang, X.; and Marques{-}Silva, J. 2024.
\newblock On Correcting {SHAP} Scores.
\newblock \emph{CoRR}, abs/2405.00076.

\bibitem[{Lundberg and Lee(2017)}]{lundberg-nips17}
Lundberg, S.~M.; and Lee, S. 2017.
\newblock A Unified Approach to Interpreting Model Predictions.
\newblock In \emph{NeurIPS}, 4765--4774.

\bibitem[{Marques{-}Silva(2022)}]{ms-rw22}
Marques{-}Silva, J. 2022.
\newblock Logic-Based Explainability in Machine Learning.
\newblock In \emph{Reasoning Web}, 24--104.

\bibitem[{Marques{-}Silva(2024)}]{ms-isola24}
Marques{-}Silva, J. 2024.
\newblock Logic-Based Explainability: Past, Present and Future.
\newblock In \emph{ISoLA}, 181--204.

\bibitem[{Marques{-}Silva and Huang(2024)}]{msh-cacm24}
Marques{-}Silva, J.; and Huang, X. 2024.
\newblock Explainability Is \emph{Not} a Game.
\newblock \emph{Commun. {ACM}}, 67(7): 66--75.

\bibitem[{Marques{-}Silva and Ignatiev(2022)}]{msi-aaai22}
Marques{-}Silva, J.; and Ignatiev, A. 2022.
\newblock Delivering Trustworthy {AI} through Formal {XAI}.
\newblock In \emph{AAAI}, 12342--12350.

\bibitem[{Merrick and Taly(2020)}]{taly-cdmake20}
Merrick, L.; and Taly, A. 2020.
\newblock The Explanation Game: Explaining Machine Learning Models Using
  {S}hapley Values.
\newblock In \emph{CDMAKE}, 17--38.

\bibitem[{Mishra(2023)}]{mishra-bk23}
Mishra, P. 2023.
\newblock \emph{Explainable {AI} Recipes}.
\newblock Apress.
\newblock ISBN 978-1-4842-9029-3.

\bibitem[{Molnar(2023)}]{molnar-bk23}
Molnar, C. 2023.
\newblock \emph{Interpreting Machine Learning Models With {SHAP}}.
\newblock Lulu.com.
\newblock ISBN 979-8857734445.

\bibitem[{Mothilal et~al.(2021)Mothilal, Mahajan, Tan, and
  Sharma}]{sharma-aies21}
Mothilal, R.~K.; Mahajan, D.; Tan, C.; and Sharma, A. 2021.
\newblock Towards Unifying Feature Attribution and Counterfactual Explanations:
  Different Means to the Same End.
\newblock In \emph{AIES}, 652--663.

\bibitem[{Robinson(2003)}]{robinson-bk03}
Robinson, D.~J. 2003.
\newblock \emph{An introduction to abstract algebra}.
\newblock Walter de Gruyter.

\bibitem[{Roth(1988)}]{roth-bk88}
Roth, A.~E. 1988.
\newblock \emph{The {S}hapley value: essays in honor of {L}loyd {S}.
  {S}hapley}.
\newblock Cambridge University Press.

\bibitem[{Shapley(1953)}]{shapley-ctg53}
Shapley, L.~S. 1953.
\newblock A value for $n$-person games.
\newblock \emph{Contributions to the Theory of Games}, 2(28): 307--317.

\bibitem[{Shapley and Shubik(1954)}]{shapley-apsr54}
Shapley, L.~S.; and Shubik, M. 1954.
\newblock A method for evaluating the distribution of power in a committee
  system.
\newblock \emph{American political science review}, 48(3): 787--792.

\bibitem[{Strumbelj and Kononenko(2010)}]{kononenko-jmlr10}
Strumbelj, E.; and Kononenko, I. 2010.
\newblock An Efficient Explanation of Individual Classifications using Game
  Theory.
\newblock \emph{J. Mach. Learn. Res.}, 11: 1--18.

\bibitem[{Strumbelj and Kononenko(2014)}]{kononenko-kis14}
Strumbelj, E.; and Kononenko, I. 2014.
\newblock Explaining prediction models and individual predictions with feature
  contributions.
\newblock \emph{Knowl. Inf. Syst.}, 41(3): 647--665.

\bibitem[{Sundararajan and Najmi(2020)}]{najmi-icml20}
Sundararajan, M.; and Najmi, A. 2020.
\newblock The Many {S}hapley Values for Model Explanation.
\newblock In \emph{ICML}, 9269--9278.

\bibitem[{Szegedy et~al.(2014)Szegedy, Zaremba, Sutskever, Bruna, Erhan,
  Goodfellow, and Fergus}]{szegedy-iclr14}
Szegedy, C.; Zaremba, W.; Sutskever, I.; Bruna, J.; Erhan, D.; Goodfellow,
  I.~J.; and Fergus, R. 2014.
\newblock Intriguing properties of neural networks.
\newblock In \emph{ICLR}.

\bibitem[{Valiant(1979)}]{valiant-sjc79}
Valiant, L.~G. 1979.
\newblock The Complexity of Enumeration and Reliability Problems.
\newblock \emph{{SIAM} J. Comput.}, 8(3): 410--421.

\bibitem[{{Van den Broeck} et~al.(2021){Van den Broeck}, Lykov, Schleich, and
  Suciu}]{vandenbroeck-aaai21}
{Van den Broeck}, G.; Lykov, A.; Schleich, M.; and Suciu, D. 2021.
\newblock On the Tractability of {SHAP} Explanations.
\newblock In \emph{AAAI}, 6505--6513.

\bibitem[{{Van den Broeck} et~al.(2022){Van den Broeck}, Lykov, Schleich, and
  Suciu}]{vandenbroeck-jair22}
{Van den Broeck}, G.; Lykov, A.; Schleich, M.; and Suciu, D. 2022.
\newblock On the Tractability of {SHAP} Explanations.
\newblock \emph{J. Artif. Intell. Res.}, 74: 851--886.

\bibitem[{Watson et~al.(2021)Watson, Gultchin, Taly, and Floridi}]{taly-uai21}
Watson, D.~S.; Gultchin, L.; Taly, A.; and Floridi, L. 2021.
\newblock Local explanations via necessity and sufficiency: unifying theory and
  practice.
\newblock In \emph{UAI}, volume 161, 1382--1392.

\bibitem[{Wu, Wu, and Barrett(2023)}]{barrett-nips23}
Wu, M.; Wu, H.; and Barrett, C.~W. 2023.
\newblock {VeriX}: Towards Verified Explainability of Deep Neural Networks.
\newblock In \emph{NeurIPS}.

\bibitem[{Yan and Procaccia(2021)}]{procaccia-aaai21}
Yan, T.; and Procaccia, A.~D. 2021.
\newblock If You Like {S}hapley Then You'll Love the Core.
\newblock In \emph{AAAI}, 5751--5759.

\bibitem[{Young(1985)}]{young-ijgt85}
Young, H.~P. 1985.
\newblock Monotonic solutions of cooperative games.
\newblock \emph{International Journal of Game Theory}, 14: 65--72.

\bibitem[{Young et~al.(2019)Young, Booth, Simpson, Dutton, and
  Shrapnel}]{shrapnel-corr19}
Young, K.; Booth, G.; Simpson, B.; Dutton, R.; and Shrapnel, S. 2019.
\newblock Deep neural network or dermatologist?
\newblock \emph{CoRR}, abs/1908.06612.

\bibitem[{Yu et~al.(2023)Yu, Farr, Ignatiev, and Stuckey}]{ignatiev-corr23b}
Yu, J.; Farr, G.; Ignatiev, A.; and Stuckey, P.~J. 2023.
\newblock Anytime Approximate Formal Feature Attribution.
\newblock \emph{CoRR}, abs/2312.06973.

\bibitem[{Yu et~al.(2024)Yu, Farr, Ignatiev, and Stuckey}]{ignatiev-sat24}
Yu, J.; Farr, G.; Ignatiev, A.; and Stuckey, P.~J. 2024.
\newblock Anytime Approximate Formal Feature Attribution.
\newblock In \emph{SAT}, 30:1--30:23.

\bibitem[{Yu, Ignatiev, and Stuckey(2023)}]{ignatiev-corr23a}
Yu, J.; Ignatiev, A.; and Stuckey, P.~J. 2023.
\newblock On Formal Feature Attribution and Its Approximation.
\newblock \emph{CoRR}, abs/2307.03380.

\end{thebibliography}
}{

}

\clearpage

\section{Supplemental Materials} \label{app:annexes}

%

\subsection{Norm $\pnorm{p}$} \label{ssec:ddbc}
%
The distance between two vectors $\mbf{v}$ and $\mbf{u}$ is denoted by
$\lVert\mbf{v}-\mbf{u}\rVert$, and the actual definition depends on
the norm being considered.
Different norms $\pnorm{p}$ can be considered.
For $p\ge1$, the $p$-norm is defined as follows~\citep{horn-bk12}:
\begin{equation}
  \begin{array}{lcl}
    \lVert\mbf{x}\rVert_{p} & {:=} &
    \left(\sum\nolimits_{i=1}^{m}|x_i|^{p}\right)^{\sfrac{1}{p}}
  \end{array}
\end{equation}

Let $d_i=1$ if $x_i\not=0$, and let $d_i=0$ otherwise. Then, for
$p=0$, we define the 0-norm, $\pnorm{0}$, as
follows~\citep{robinson-bk03}:
\begin{equation}
  \begin{array}{lcl}
    \lVert\mbf{x}\rVert_{0} & {:=} &
    \sum\nolimits_{i=1}^{m}d_i
  \end{array}
\end{equation}

In general, for $p\ge1$, $\pnorm{p}$ denotes the Minkowski distance.
Well-known special cases include
the Manhattan distance $\pnorm{1}$,
the Euclidean distance $\pnorm{2}$, and
the Chebyshev distance $\pnorm{\infty}=\lim_{p\to\infty}l_p$.
$\pnorm{0}$ denotes the Hamming distance.
%
%

\subsection{Deterministic Decomposable Boolean Circuits (DDBCs)}
%
For some complexity results, we analyze
DDBCs~\cite{barcelo-aaai21,barcelo-jmlr23}\footnote{%
The definition of DDBC mimics the one 
in~\cite{barcelo-jmlr23}.}.
A boolean circuit $C$ is defined on a set of (input) variables $X$ and it
is represented as a directed acyclic graph, where each node is
referred to as a gate, and where (i) a node with no input edges is a
either a variable gate, and takes a label from $X$, or it is a
constant gate, and takes a label from $\{0,1\}$; (ii) a node with
incoming edges is a either a AND, OR or NOT logic gate, where NOT
gates have exactly one input; (iii) exactly one node has no output
edges, and denotes the output gate of $C$.
Given some circuit $C_g$, $\tn{var}(C_g)$ denotes the set of elements
$x\in{X}$ such that some variable gate node of $C_g$ is labeled with
$x$.
A DDBC is a boolean circuit where OR gates are \emph{deterministic}
and AND gates are \emph{decomposable}. A 2-input OR gate,
$g=\tn{OR}(g_1,g_2)$ is deterministic if for any assignment to the
inputs of the circuit, the inputs of the gate are \emph{not} both
assigned value 1. A 2-input AND gate, 
$g=\tn{AND}(g_1,g_2)$, is decomposable if $\tn{var}(C_{g_1})$
is disjoint from $\tn{var}(C_{g_2})$.
%
It is well-known that any DDBC can be converted into a DDBC
containing only 2-input AND and OR gates in polynomial time~\cite{barcelo-jmlr23}.
DDBCs generalize 
\emph{deterministic decomposable negation
normal form} (d-DNNF) circuits~\cite{darwiche-jair02}.
Furthermore, we consider a recent generalization of DDBCs where inputs
are allowed to take multi-valued discrete
values~\cite{barcelo-jmlr23}.


%

\subsection{Additional examples}

\paragraph{Difference in SHAP scores for example DT.}
%
\cref{ex02} shows the similarity predicate $\similar_{2}$ for an
example classifier $\kappa_{2}$, represented by a decision tree
(DT).
%
%
\begin{figure*}[t]
  \begin{subfigure}[b]{0.475\linewidth}
    \centering
    \scalebox{0.875}{
%
\forestset{
  BDT/.style={
    for tree={
      l=1.5cm,s sep=1.15cm,
      if n children=0{}{circle}, 
      draw=midblue,
      text=midblue,
      edge={
        my edge
      },
      edge=thick,
    }
  },
}
\begin{forest}
  BDT
  [{$x_1$}, label={[yshift=-6.875ex]{{\tiny1}}} 
    [{$x_4$}, label={[yshift=-6.875ex]{{\tiny2}}}, 
      edge label={node[midway,left,xshift=-0.5pt] {{\scriptsize$\in\{0\}$}}}
      [{$x_2$}, label={[yshift=-6.875ex]{{\tiny4}}}, 
        edge label={node[midway,left,xshift=-1.5pt] {{\scriptsize$\in\{1\}$}}}
        [{$x_3$}, label={[yshift=-6.875ex]{{\tiny6}}}, 
          edge label={node[midway,left,xshift=-1.5pt] {{\scriptsize$\in\{0\}$}}}
          [\ncolor{\textbf{5}}, label={[yshift=-5.25ex]{{\tiny8}}},
            edge label={node[midway,left,xshift=-0.5pt] {{\scriptsize$\in\{0\}$}}}, rectangle, fill={torange1!10} ]
          [\ncolor{\textbf{2}}, label={[yshift=-5.25ex]{{\tiny9}}},
            edge label={node[midway,right,xshift=-0.575pt] {{\scriptsize$\in\{1\}$}}}, rectangle, fill={torange1!10} ]
        ]
        [{$x_3$}, label={[yshift=-6.875ex]{{\tiny7}}}, 
          edge label={node[midway,right,xshift=0.25pt] {{\scriptsize$\in\{1\}$}}}
          [\ncolor{\textbf{4}}, label={[yshift=-5.25ex]{{\tiny10}}},
            edge label={node[midway,left,xshift=-0.5pt] {{\scriptsize$\in\{0\}$}}}, rectangle, fill={torange1!10} ]
          [\ncolor{\textbf{9}}, label={[yshift=-5.25ex]{{\tiny11}}},
            edge label={node[midway,right,xshift=-0.575pt] {{\scriptsize$\in\{1\}$}}}, rectangle, fill={torange1!10} ]
        ]
      ]
      [\ncolor{\textbf{0}}, label={[yshift=-5.25ex]{{\tiny5}}},
        edge label={node[midway,right,xshift=-0.5pt] {{\scriptsize$\in\{0,2\}$}}}, rectangle, fill={torange1!10} ]
    ]
    [\ncolor{\textbf{1}}, label={[yshift=-5.25ex]{{\tiny3}}},
      edge={very thick,draw=tblue2}, edge label={node[midway,right,xshift=0.5pt] {{\scriptsize$\in\{1\}$}}}, rectangle, fill={torange1!10} ]
  ]
\end{forest}
    }
    \caption{Decision tree.}
    \label{ex02:dt}
  \end{subfigure}
  \begin{subfigure}[b]{0.475\linewidth}  
    \centering
    \scalebox{0.875}{
%
\forestset{
  BDT/.style={
    for tree={
      l=1.5cm,s sep=1.15cm,
      if n children=0{}{circle}, 
      draw=midblue,
      text=midblue,
      edge={
        my edge
      },
      edge=thick,
    }
  },
}
\begin{forest}
  BDT
  [{$x_1$}, label={[yshift=-6.875ex]{{\tiny1}}} 
    [{$x_4$}, label={[yshift=-6.875ex]{{\tiny2}}}, 
      edge label={node[midway,left,xshift=-0.5pt] {{\scriptsize$\in\{0\}$}}}
      [{$x_2$}, label={[yshift=-6.875ex]{{\tiny4}}}, 
        edge label={node[midway,left,xshift=-1.5pt] {{\scriptsize$\in\{1\}$}}}
        [{$x_3$}, label={[yshift=-6.875ex]{{\tiny6}}}, 
          edge label={node[midway,left,xshift=-1.5pt] {{\scriptsize$\in\{0\}$}}}
          [\rhlight{\textbf{0}}, label={[yshift=-5.25ex]{{\tiny8}}},
            edge label={node[midway,left,xshift=-0.5pt] {{\scriptsize$\in\{0\}$}}}, rectangle, fill={tred1!10} ]
          [\rhlight{\textbf{0}}, label={[yshift=-5.25ex]{{\tiny9}}},
            edge label={node[midway,right,xshift=-0.575pt] {{\scriptsize$\in\{1\}$}}}, rectangle, fill={tred1!10} ]
        ]
        [{$x_3$}, label={[yshift=-6.875ex]{{\tiny7}}}, 
          edge label={node[midway,right,xshift=0.25pt] {{\scriptsize$\in\{1\}$}}}
          [\rhlight{\textbf{0}}, label={[yshift=-5.25ex]{{\tiny10}}},
            edge label={node[midway,left,xshift=-0.5pt] {{\scriptsize$\in\{0\}$}}}, rectangle, fill={tred1!10} ]
          [\rhlight{\textbf{0}}, label={[yshift=-5.25ex]{{\tiny11}}},
            edge label={node[midway,right,xshift=-0.575pt] {{\scriptsize$\in\{1\}$}}}, rectangle, fill={tred1!10} ]
        ]
      ]
      [\rhlight{\textbf{0}}, label={[yshift=-5.25ex]{{\tiny5}}},
        edge label={node[midway,right,xshift=-0.5pt] {{\scriptsize$\in\{0,2\}$}}}, rectangle, fill={tred1!10} ]
    ]
    [\dghlight{\textbf{1}}, label={[yshift=-5.25ex]{{\tiny3}}},
      edge label={node[midway,right,xshift=0.5pt] {{\scriptsize$\in\{1\}$}}}, rectangle, fill={tgreen1!10} ]
  ]
\end{forest}
    }
    \caption{similarity predicate.}
    \label{ex02:cf}
  \end{subfigure}
  \caption{Example decision tree (DT), representing classifier
    $\kappa_2$, and respective similarity predicate $\similar_{2}$. The
  target sample is $((1,1,1,1),1)$.} \label{ex02}
\end{figure*}
Given the obtained characteristic function, \cref{tab:shap_vs_sshap}
shows the computed SHAP scores, obtained with both
SHAP~\cite{lundberg-nips17} and sSHAP. (Given the simple classifiers
being considered, both SHAP and sSHAP obtain the exact SHAP scores.)
\begin{table}[t]
  \centering
  \renewcommand{\tabcolsep}{0.4125em}
  \begin{tabular}{cccccc} \toprule
    Classifier & $\sv(1)$  & $\sv(2)$  & $\sv(3)$  & $\sv(4)$ & Rank \\
    \toprule
    $\kappa_{2}$ & 0.000 & 0.111 & 0.056 & -0.500 &
    $\langle4,2,3,1\rangle$ \\
    $\similar_{2}$ & 0.500 & 0 & 0 & 0 &
    $\langle1,2:3:4\rangle$ \\
    \bottomrule
  \end{tabular}
  \caption{SHAP scores $\svn{E}$ and $\svn{S}$, respectively for
    $\kappa_{2}$ and $\similar_{2}$.} 
  \label{tab:shap_vs_sshap}
\end{table}
For this concrete example and instance, the results confirm that the
new characteristic function $\cfn{s}$ enables obtaining SHAP scores
that are not misleading.
As we shown next, the same situation is observed for other
classifiers.


%

\subsection{Analysis of Ranking Models}
\label{app:rms}

This section shows how the framework proposed in the paper can also be
used for computing abductive explanations (and so formal SHAP scores)
when the ML model implements a ranking model.
A well-known instantiation of a ranking model is the softmax layer for
a neural network (NN).
Throughout this section, the analogy with a softmax layer will help
understanding the definitions used.
We assume a set of indices $\fml{U}=\{1,\ldots,k_{\fml{U}}\}$ each
indexing one output of the ranking model. The outputs take values from
a set $\mbb{C}$. The selection of each index maps to a categorical or  
ordinal value taken from $\mbb{T}$. The output is selected to be the
index computing the largest value among the indices in $\fml{U}$.
Given the above, the following functions are considered:
\begin{enumerate}[nosep]
\item $\rho:\fml{U}\times\mbb{F}\to\mbb{C}$;
\item $\pi:\fml{U}\to\mbb{T}$; and
\item $\eta:\mbb{F}\to\mbb{T}$.
\end{enumerate}

Moreover, we introduce an auxiliary function $\chi:\mbb{F}\to\fml{U}$,
$\chi(\mbf{x})=\argmax\{\rho(j,\mbf{x})\,|\,j\in\fml{U}\}$,
i.e. $\chi(\mbf{x})$ gives the index of the ranking model predicting
the largest value.
Finally, the prediction is given by
$\eta(\mbf{x})=\pi(\chi(\mbf{x}))$. (Observe that $\chi$ serves only
to simplify the notation.)
When referring to a ranking model $\fml{M}_{e}$, a tuple
$(\fml{F},\mbb{F},\fml{U},\mbb{C},\mbb{T},\rho,\pi,\eta)$ is assumed.

Let $\iota(t):\mbb{T}\to\fml{U}$ assign to each prediction $t$ the
index of the ranking model that predicts $t$.
For some input $\mbf{x}$, the output is distinguished from that of
$\mbf{v}$ if the following predicate,
\[
\tau(\mbf{x};\fml{E},\delta) ~~ := ~~ [\chi(\mbf{x}) = \chi(\mbf{v})]
\]
holds true. (We could also account for a difference of $\delta$ of the
model's output, but that is beyond the goals of this paper.)

Thus, the definitions of abductive (or contrastive) explanations and
the (novel) definitions of SHAP scores also hold in the case when a
ranking model is assumed.

\end{document}